\documentclass[lettersize,journal]{IEEEtran}
\usepackage{amsmath,amsfonts}
\usepackage{algorithmic}
\usepackage{algorithm}
\usepackage{array}
\usepackage[caption=false,font=normalsize,labelfont=sf,textfont=sf]{subfig}
\usepackage{textcomp}
\usepackage{stfloats}
\usepackage{url}
\usepackage{verbatim}
\usepackage{graphicx}
\usepackage{cite}
\usepackage[colorlinks,citecolor=blue,linkcolor=blue]{hyperref}
\newtheorem{theorem}{Theorem}

\newtheorem{lemma}{Lemma}
\newtheorem{assumption}{Assumption}
\newtheorem{proposition}{Proposition}

\newtheorem{problem}{Problem}
\newtheorem{definition}{Definition}
\newtheorem{remark}{Remark}

\newtheorem{proof}{Proof}
\newcommand{\be}{\begin{eqnarray}}
\newcommand{\ee}{\end{eqnarray}}
\newcommand{\bq}{\begin{equation}}
\newcommand{\eq}{\end{equation}}
\newcommand{\e}{\mathbb{E}}
\newcommand{\q}{\quad}
\newcommand{\qq}{\qquad}

\begin{document}

\title{Entropy-Regularized Reinforcement Learning for\\ Linear-Quadratic Stackelberg Differential Games in \\Regime-Switching Diffusion Models}
\author{
\thanks{This study was funded by the National Key R \& D Program of China (2022YFA1007900), Natural Science Foundation of Anhui Province (2408085MA019), National Natural Science Foundation of China (12271171), Shanghai Philosophy Social Science Planning Office Project (2022ZJB005), and Fundamental Research Funds for the Central Universities (2024QKT008).
(Corresponding author: Lin Xu.)}
Congde Hu, \thanks{Congde Hu is with the School of Mathematics and Statistics, Anhui Normal University, Wuhu, Anhui  241002, China (e-mail:hcd.math@gmail.com).}
\and Danping Li, \thanks{Danping Li is with the School of Statistics, Key Laboratory of Advanced Theory and Application in Statistics and Data Science-MOE,
Faculty of Economics and Management, East China Normal University, Shanghai, China (e-mail:dpli@fem.ecnu.edu.cn).}
\and Lin Xu, \thanks{Lin Xu is with the School of Mathematics and Statistics, Anhui Normal University, Wuhu, Anhui  241002, China (e-mail:xulinahnu@gmail.com).}
\and Wenying Xu, \thanks{Wenying Xu is with the School of Mathematics, Southeast University, Nanjing, Jiangsu 211189, China (e-mail: wyxu@seu.edu.cn ).}\IEEEmembership{Senior Member, IEEE}
}
\markboth{IEEE Transactions on xxx}{Hu \MakeLowercase{\textit{et al.}}: Entropy-Regularized Reinforcement Learning for Linear-Quadratic Stackelberg Differential Games in Continuous-Time Diffusion Models}

\IEEEpubid{0000--0000/00\$00.00~\copyright~2025 IEEE}

\maketitle

\begin{abstract}
Stackelberg differential games (SDGs) provide a powerful framework for hierarchical decision-making in stochastic and continuous-time environments, yet their solution remains computationally challenging due to the complexity of traditional dynamic programming and Hamilton-Jacobi-Bellman-Isaacs (HJBI) methods, especially in high-dimensional systems. This paper proposes an entropy-regularized reinforcement learning (ERRL) approach for linear-quadratic SDGs (LQ-SDGs) within a continuous-time diffusion framework governed by Markovian regime switching. The key innovation lies in deriving exploratory weakly-coupled HJBI equations with entropy regularization, which promotes stochastic policies that actively avoid suboptimal equilibria---a limitation of classical SDG methods. Neural networks are integrated to approximate regime-dependent value functions and solve high-dimensional partial differential equations (PDEs) efficiently, while a novel sampling technique enhances computational tractability. Numerical results demonstrate the effectiveness of the framework compared to conventional approaches, particularly in escaping suboptimal traps through exploratory policies. The study highlights the critical role of entropy regularization and neural network approximations in achieving robust solutions for hierarchical decision-making problems under abrupt environmental shifts.
\end{abstract}


\section{Introduction}
\subsection{Background and motivation}
\IEEEPARstart{S}{}tackelberg differential games (SDGs)
represent a cornerstone in the field of game theory, offering a robust framework for modeling hierarchical decision-making processes. The concept of Stackelberg games was first introduced by \cite{stackelberg1934market}, which sought to describe economic scenarios where one player (the leader) has a dominant role over another (the follower). This hierarchical structure has since been extended to continuous-time and stochastic settings, giving rise to SDGs, which have found applications in a wide range of disciplines, including economics \cite{basar1999dynamic}, engineering \cite{10945872}, robotics \cite{xing2024denial}, supply chain optimization \cite{pakseresht2020toward}, autonomous vehicles \cite{10440182}, financial markets \cite{sockin2023decentralization}.

Despite their versatility, solving SDGs remains a challenging task due to the complexity of deriving optimal strategies in continuous-time and stochastic environments. Traditional methods often rely on dynamic programming principles (DPP) and the Hamilton-Jacobi-Bellman-Isaacs (HJBI) equations \cite{yong1999stochastic}. However, these approaches can become computationally intractable in high-dimensional or nonlinear systems, limiting their applicability to real-world problems.

Reinforcement learning (RL) has emerged as a powerful paradigm for solving complex decision-making problems, particularly in scenarios where traditional methods fall short. Unlike classical optimization techniques, RL leverages data-driven approaches to learn optimal policies through interaction with the environment. This makes RL particularly well-suited for problems involving uncertainty, incomplete information, and dynamic environments \cite{sutton2018reinforcement,han2020spatial}. entropy-regularized reinforcement learning (ERRL)---a variant of RL that incorporates an entropy term into the objective function, has gained significant attention in recent years. The entropy term encourages exploration by promoting stochasticity in the policy, thereby \emph{preventing premature convergence to suboptimal solutions}. This approach has been shown to improve the robustness and stability of learning algorithms, making it particularly effective in high-dimensional and continuous action spaces \cite{mceneaney2007curse}.
\IEEEpubidadjcol
The integration of entropy regularization into RL has led to the development of several state-of-the-art algorithms, such as soft actor-critic (SAC) \cite{haarnoja2018soft} and maximum entropy deep RL \cite{zhao2019maximum}. These algorithms have demonstrated superior performance in a wide range of applications, including robotics, game theory, and control systems. For instance, \cite{miao2025} demonstrated the effectiveness of ERRL in multi-agent systems, where the stochasticity of policies facilitated better coordination among agents.

In the context of SDGs, ERRL offers several advantages. First, the stochastic nature of entropy-regularized policies aligns well with the hierarchical structure of SDGs, enabling the leader and follower to adapt their strategies dynamically. Second, the exploration-promoting properties of entropy regularization enhance the robustness of learning algorithms, particularly in stochastic environments. Finally, the probabilistic interpretation of entropy-regularized policies facilitates derivation of smooth and interpretable strategies, which are crucial for real-world applications \cite{zhao2019maximum}.

Despite these advancements, the integration of entropy regularization into RL for SDGs remains underexplored, particularly in the context of continuous-time diffusion models. Diffusion models provide a natural framework for modeling stochastic dynamics in continuous time and have recently gained attention in RL due to their ability to handle uncertainty and noise effectively \cite{wang2019exploration,tang2022exploratory}. Furthermore, many real-world systems are not only subject to continuous random fluctuations but also experience abrupt structural changes due to macroeconomic shifts, environmental variations, or component failures. Markovian regime switching models provide a powerful mathematical tool to capture such sudden transitions. Integrating regime switching into the ERRL framework for SDGs is therefore practically significant but mathematically demanding, as it introduces weakly-coupled systems of equations.

\subsection{Main contributions of this paper}
This paper addresses the aforementioned challenges by proposing a novel framework for solving linear-quadratic SDGs (LQ-SDGs) by using ERRL within a continuous-time diffusion framework featuring Markovian regime switching. LQ-SDGs have evolved since their inception, bridging deterministic and stochastic frameworks with applications in economics, engineering, and finance. Early works by \cite{yong2002leader} established deterministic LQ-SDGs, while stochastic extensions incorporated Brownian motion, yielding Riccati-based solutions under adapted strategies. Recent advances address jump-diffusion systems, where \cite{lv2023linear} investigated LQ-SDGs for mean-field switching diffusions, \cite{sun2023zero} studied zero-sum LQ-SDGs, \cite{moon2023linear} derived equilibrium strategies via coupled Riccati equations with L\'{e}vy processes, generalizing results to non-Gaussian noise, \cite{li2024closed} provided closed-loop solvability of mean-field LQ-SDGs. Contemporary research focuses on mean-field LQ-SDGs \cite{yong2013linear,bensoussan2017linear} and time-inconsistent problems \cite{wang2023time}, leveraging backward stochastic differential equations (BSDEs) and leader-follower asymmetry.

The work closely related to the problem discussed in this paper is \cite{li2019two}, which established two coupled forms of the HJBI equations and proved that the solutions of these HJBI equations not only stabilize the system but also constitute the Stackelberg equilibrium strategies. Due to the difficulties in solving the HJBI equations, they also developed an iterative algorithm and conducted simulation examples, which can be found in detail in \cite{li2019two}. However, on the one hand, the controlled system therein does not take into account the impact of continuous random disturbances and sudden environmental shifts (i.e., regime switching). On the other hand, finding the optimal strategies often leads to the so-called \lq\lq curse of optimality." As pointed out in \cite{zhou2021curse}, we strive to seek optimality, but often find ourselves trapped in bad \lq\lq optimal" solutions that are either local optimizers, or too rigid to leave any room for errors, or based on wrong models. A way to break this \lq\lq curse of optimality" is to engage exploration through randomization. In this paper, we will present a dynamic system equation that takes into account additional Brownian motion and Markovian regime switching, and has an exploration mechanism, as well as performance functions that encourage exploration.

The main contributions of this work are summarized as follows:
\begin{itemize}
\item An exploratory system of weakly-coupled HJBI equations for SDGs with Markovian regime switching is derived by applying the DPP. This derivation establishes a rigorous theoretical foundation for the proposed approach while incorporating an entropy regularization term, ensuring that the resulting policies remain stochastic and exploration-promoting across different system regimes.
\item Coupled distributional optimal policies for both the leader and follower are derived, ensuring robustness and adaptability in stochastic environments subject to structural jumps. A partially model-free policy improvement algorithm (PIA) is designed to approximate the regime-dependent value functions within an ERRL framework, enhancing exploration and convergence.
\item The algorithm integrates a novel sampling technique that exploits the structure of the diffusion model, leading to improved computational efficiency. For fast approximating to the value functions across multiple discrete regimes, a neural network approximation architecture for high-dimensional partial differential equations (PDEs) is incorporated into the procedure design.
\end{itemize}

The remainder of this paper is organized as follows. Section \ref{Section 2} introduces the learning framework, including regime-switching dynamics, and formulates the problem. Section \ref{Section 3} derives the coupled HJBI equations for the ERRL LQ-SDGs and characterizes the equilibrium strategies. Section \ref{Section 4} develops a PIA to approximate the value function and compute these strategies. Finally, Section \ref{Section 5} provides numerical examples to demonstrate algorithm convergence, the effect of temperature parameter,  and specifically how the framework avoids the \lq\lq local optimum trap\rq\rq.
\section*{Notations}
 $\mathbb{R}^n$ and $\mathbb{S}^n$ denote $n$-dimensional vectors and symmetric matrices, respectively. For a matrix $A$, its transpose, inverse, and trace are $A^{T}$, $A^{-1}$, and $tr(A)$. $A>0$ ($A\geq0$) means $A$ is positive (semi-)definite, with square root $A^{1/2} = UD^{1/2}U^T$. $|\cdot|$ and $||B||$ are the absolute value and Euclidean norm. $\delta$ is the variational operator, and $\e_{x,i}[\cdot]$ is the conditional expectation given state $x$ and regime $i$. For a function $f(x)$, $\nabla f(x)$ and $\Delta f(x)$ are its gradient and Hessian. Finally, $a \land b$ is the minimum of $a$ and $b$, and $\mathcal{N}(\mu, \Sigma)$ represents the Gaussian distribution.
\section{Formulation and problem}\label{Section 2}
\setcounter{equation}{0}
\renewcommand{\theequation }{2.\arabic{equation}}
This section aims to construct the two-player LQ-SDGs problem with learning. To clarify our objectives, we first introduce a classical two-player Stackelberg game problem. The content in subsection \ref{sub21} mainly comes from \cite{li2019two}. For the convenience of narration, the problem studied by \cite{li2019two} is referred to as the classical two-player Stackelberg game problem.
\subsection{Classical two-player Stackelberg game problem}\label{sub21}
Consider a linear dynamic system with two players
\begin{equation}\label{2.1}
dx_s = \left(Ax_s + B_1u_s + B_2v_s\right)ds,\,s>0
\end{equation}
where $x=\{x_s,s>0\} \in \mathbb{R}^n$ is the measurable system state, $u=\{u_s,s>0\}$ $(resp.,v=\{v_s,s>0\})$$\in \mathbb{R}^p$ is the policy of player $I$ (resp., player $II$), $B_1$ and $B_2$ are matrices with compatible dimensions. Player $I$ is the leader, holding a dominant position that enables it to anticipate player $II$'s response and make its decision $u$ first. Next, player $II$  (the follower) observes $u$ and responds with action $v$. Given this decision-making structure, the performance function for each player is defined as
\begin{equation}\label{2.2}
J^{cl}_k(x, u, v) = \int_0^{\infty} r_k(x_s, u_s, v_s) \, ds, \, k = 1, 2,
\end{equation}
with
\begin{equation*}
r_1(x, u, v) = x^T Q_1 x + (u + \theta_1 v)^T R_1 (u + \theta_1 v),
\end{equation*}
\begin{equation*}
r_2(x, u, v) = x^T Q_2 x + (v + \theta_2 u)^T R_2 (v + \theta_2 u),
\end{equation*}
and $Q_k \geq 0, R_k > 0$, and $\theta_k \in (0,1)$ for $k = 1, 2$.

Denote the admissible sets of the two players by $\mathcal{U}_{cl}$ and $\mathcal{V}_{cl}$ respectively with
\begin{enumerate}
    \item[(i)] $\int_{0}^{t}||u_s||^{2}ds<\infty$, $\int_{0}^{t}||v_s||^{2}ds<\infty$, $\forall t\geq0$;
    \item[(ii)] with $\{x_s,s\geq0\}$ solving \eqref{2.1}, $\int_0^{\infty}|r_k(x_s,u_s,v_s)|ds<\infty$, $k=1,2$.
\end{enumerate}
\begin{definition}
The classical value functions of player $II$  and player $I$  are given by
\begin{equation}
V^{cl}_2(x;u) = \min_{v\in\mathcal{V}_{cl}}J^{cl}_2(x, u,v)
\end{equation}
and
\begin{equation}
V^{cl}_1(x) = \min_{u\in\mathcal{U}_{cl}} J^{cl}_1(x, u,v^*(u)),
\end{equation}
where $v^*(u)= \arg\min_{v\in\mathcal{V}_{cl}}J^{cl}_2(x, u,v)$.
\end{definition}
For the  dynamic system \eqref{2.1} and performance function \eqref{2.2}, the objective of classical SDG aims to design appropriate policies $u^*$ and $v^*$ satisfying the following \emph{Stacklelberg equilibrium} conditions,
\begin{equation}\label{2.5}
\left\{\begin{array}{ll}
    J^{cl}_1(x, u^*, v^*(u^*))\leq J^{cl}_1(x, u, v^*(u)), \,\forall u\in\mathcal{U}_{cl}; \\
    J^{cl}_2(x, u, v^*(u))\leq J^{cl}_2(x, u, v), \, \forall v\in\mathcal{V}_{cl},
\end{array} \right.
\end{equation}
where $v^*(u)= \arg\min_{v\in\mathcal{V}_{cl}}J^{cl}_2(x, u,v)$.

\subsection{Exploratory two-player LQ-SDGs framework}\label{sub22}
System \eqref{2.1} does not incorporate stochastic effects and environmental variations, which motivates us to introduce Brownian motion and regime switching in our analysis.

Let $\alpha = \{\alpha_s, s \ge 0\}$ be a continuous-time Markov chain defined on a complete probability space $(\Omega, \mathcal{F}, \{\mathcal{F}_s\}_{s \ge 0}, \mathbb{P})$ with a finite state space $\mathcal{M} = \{1, 2, \dots, m\}$. The generator matrix of this Markov chain is defined as $\Gamma = (q_{ij})_{m \times m}$, where $q_{ij} \ge 0$ ($i \neq j$) represents the transition rate from regime $i$ to regime $j$, satisfying $q_{ii} = -\sum_{j=1,j \neq i}^m q_{ij}$.
The dynamical system with regime switching is given by
\begin{equation}\label{2.6}
dx_s = b(x_s, \alpha_s, u_s, v_s)ds + \sigma(x_s, \alpha_s, u_s, v_s)dW_s, s > 0, 
\end{equation}
where the drift and diffusion terms are determined by the current regime $i\in\mathcal{M}$:
\[
b(x, i, u, v) := A(i)x + B_1(i)u + B_2(i)v,
\]
\[
\sigma(x, i, u, v) := C(i)x + D_1(i)u + D_2(i)v.
\]
Here, $W=\{W_s, s \ge 0\}$ is a standard $n$-dimensional Brownian motion, assumed to be independent of the Markov chain $\alpha$. All system parameters (e.g., $A(i), B_1(i)$, etc.) are constant matrices dependent on the current regime $i\in\mathcal{M}$.
Correspondingly, the classical performance function for each player also depends on the current regime $i\in\mathcal{M}$:
\[
r_1(x, i, u, v) = x^T Q_1(i) x + (u + \theta_1(i) v)^T R_1(i) (u + \theta_1(i) v),
\]
\[
r_2(x, i, u, v) = x^T Q_2(i) x + (v+ \theta_2(i) u)^T R_2(i) (v + \theta_2(i) u),
\]
where $Q_k(i) \ge 0, R_k(i) > 0, \theta_k(i) \in (0,1)$ for $k = 1, 2$.

A core principle in RL is environmental exploration through randomized actions. This involves replacing deterministic policies with probability distributions (measures) to facilitate exploration. Entropy quantifies the dispersion of these distributions: low entropy indicates concentration, while high entropy reflects greater randomness. This paper focuses on feedback distributional policies, where the random policy's distribution depends solely on the current state $x$ and regime $i$.

Let $\boldsymbol{\pi} : (x, i) \in \mathbb{R}^n \times \mathcal{M} \rightarrow \boldsymbol{\pi}(\cdot \mid x, i) \in \mathcal{P}(\mathbb{R}^p)$ (resp., $\boldsymbol{\gamma} : (x, i) \in \mathbb{R}^n \times \mathcal{M} \rightarrow \boldsymbol{\gamma}(\cdot \mid x, i) \in \mathcal{P}(\mathbb{R}^p)$) be a given stochastic feedback control, where $\mathcal{P}(\mathbb{R}^p)$ is the set of probability density functions defined on $\mathbb{R}^p$. Let the control actions $u \in \mathbb{R}^p$ and $v \in \mathbb{R}^p$ be sampled from the conditional distributions $\boldsymbol{\pi}(\cdot \mid x, i)$ and $\boldsymbol{\gamma}(\cdot \mid x, i)$, respectively. For notational brevity, when the current state $x$ and regime $i$ are clear from the context, we will often drop the conditioning and write $\boldsymbol{\pi}(u)$ and $\boldsymbol{\gamma}(v)$ to denote $\boldsymbol{\pi}(u \mid x, i)$ and $\boldsymbol{\gamma}(v \mid x, i)$. Inspired by \cite{wang2020reinforcement}, the drift, diffusion, and performance functions are then extended to an exploratory version \eqref{extend1}--\eqref{extend4} given by
\begin{equation}\label{extend1}
\tilde{b}(x,i,\boldsymbol{\pi},\boldsymbol{\gamma}):=\int_{\mathbb{R}^p}\left[\int_{\mathbb{R}^p}b(x,i,u,v)\boldsymbol{\pi}(u)\,du\right]\boldsymbol{\gamma}(v)\,dv,
\end{equation}
\begin{equation}
\begin{aligned}
&\tilde{\sigma}(x,i,\boldsymbol{\pi},\boldsymbol{\gamma})\\
&:=\left\{\int_{\mathbb{R}^p}\left[\int_{\mathbb{R}^p}\sigma(x,i,u,v)\sigma^T(x,i,u,v)\boldsymbol{\pi}(u)\,du\right]\boldsymbol{\gamma}(v)\,dv\right\}^{1/2},
\end{aligned}
\end{equation}
and
\begin{equation}\label{extend4}
\begin{aligned}
\tilde{r}_1(x,i, \boldsymbol{\pi}, \boldsymbol{\gamma}):=\int_{\mathbb{R}^p}\left[\int_{\mathbb{R}^p}r_1(x,i, u, v)\boldsymbol{\pi}(u)\,du\right]\boldsymbol{\gamma}(v)\,dv,\\
\tilde{r}_2(x,i, \boldsymbol{\pi}, \boldsymbol{\gamma}):=\int_{\mathbb{R}^p}\left[\int_{\mathbb{R}^p}r_2(x,i, u, v) \boldsymbol{\pi}(u)\,du\right]\boldsymbol{\gamma}(v)\,dv.
\end{aligned}
\end{equation}
For the reader's convenience, we postpone the derivation of \eqref{extend1}--\eqref{extend4} to Appendix \ref{Appendix A}.

Accordingly, the exploratory dynamic system for \eqref{2.6} is given by
\begin{equation}\label{2.11}
dX_s^{\boldsymbol{\pi},\boldsymbol{\gamma}}=\tilde{b}(X^{\boldsymbol{\pi},\boldsymbol{\gamma}}_s, \alpha_s,\boldsymbol{\pi}_s,\boldsymbol{\gamma}_s)ds+\tilde{\sigma}(X^{\boldsymbol{\pi},\boldsymbol{\gamma}}_s, \alpha_s,\boldsymbol{\pi}_s,\boldsymbol{\gamma}_s)dW_s,
\end{equation}
where we adopt the shorthands $\boldsymbol{\pi}_s(\cdot) := \boldsymbol{\pi}(\cdot \mid X_s^{\boldsymbol{\pi},\boldsymbol{\gamma}}, \alpha_s)$ and $\boldsymbol{\gamma}_s(\cdot) := \boldsymbol{\gamma}(\cdot \mid X_s^{\boldsymbol{\pi},\boldsymbol{\gamma}}, \alpha_s)$ to denote the policies evaluated along the state trajectories. 

Similar to \cite{wang2020continuous}, to encourage exploration, we add an entropy regularization term to the performance function. Then, with \eqref{2.11}, the exploratory performance functions of player $I$ and player $II$ are respectively defined by
\begin{equation}\label{2.12}
\begin{aligned}
J_1(x,i, \boldsymbol{\pi},\boldsymbol{\gamma}) =& \e_{x,i}\left[\int_0^{\infty}e^{-\rho_1 s} \Big[\tilde{r}_1(X_s^{\boldsymbol{\pi}, \boldsymbol{\gamma}}, \alpha_s, \boldsymbol{\pi}_s, \boldsymbol{\gamma}_s)\right.\\
&\left.+\lambda_1\int_{\mathbb{R}^p} \boldsymbol{\pi}_s(u)\ln \boldsymbol{\pi}_s(u)\,du\,\Big] ds\right],
\end{aligned}
\end{equation}
\begin{equation}\label{2.13}
\begin{aligned}
J_2(x,i, \boldsymbol{\pi},\boldsymbol{\gamma}) =& \e_{x,i}\left[\int_0^{\infty}e^{-\rho_2 s} \Big[\tilde{r}_2(X_s^{\boldsymbol{\pi}, \boldsymbol{\gamma}}, \alpha_s, \boldsymbol{\pi}_s, \boldsymbol{\gamma}_s)\right.\\
&\left.+\lambda_2\int_{\mathbb{R}^p} \boldsymbol{\gamma}_s(v)\ln \boldsymbol{\gamma}_s(v)\,dv\Big] ds\right],
\end{aligned}
\end{equation}
where $\e_{x,i}[\cdot]$ denotes the expectation conditioned on the initial state $X_0^{\boldsymbol{\pi},\boldsymbol{\gamma}}=x$ and initial regime $\alpha_0=i$. Here, $\rho_k > 0$ $(k=1,2)$ are the \emph{discount factors}, and the exogenous parameters $\lambda_k>0$ $(k=1,2)$ are the \emph{temperature parameters}, which reflect the level of exploration.
\begin{remark}
The use of entropy in ERRL provides a principled way to encourage exploration by promoting policies with higher entropy. Higher entropy signifies greater randomness. Since the objectives of player $I$ and player $II$ are to minimize the performance functions in equations \eqref{2.12} and \eqref{2.13} respectively, the negative entropy terms $\lambda_1\int_{\mathbb{R}^p} \boldsymbol{\pi}(u)\ln \boldsymbol{\pi}(u)\,du$ and $\lambda_2\int_{\mathbb{R}^p} \boldsymbol{\gamma}(v)\ln \boldsymbol{\gamma}(v)\,dv$ can be regarded as \lq\lq reward" parameters. $\lambda_1$ and $\lambda_2$ are the temperature parameters of the two players, which determine how much new information overrides old information during the learning process, i.e., how quickly the agent updates its policy and value function estimates based on the observed rewards and entropy bonus. Recent work has also focused on Choquet regularization \cite{han2023choquet} and the more general Tsallis entropy \cite{bo2024continuous,chen2025exploratory}.
\end{remark}
\begin{assumption}\label{assumption 1}
Assume that
\begin{enumerate}
\item[(i)] for each $s\geq0$, $\boldsymbol{\pi}_s\in\mathcal{P}(\mathbb{R}^p)$ and $\boldsymbol{\gamma}_s\in\mathcal{P}(\mathbb{R}^p)$ a.s.;
\item[(ii)] $\{\int_{A}\boldsymbol{\pi}_s(u)\,du, s\geq0\}$ and $\{\int_{A}\boldsymbol{\gamma}_s(v)\,dv, s\geq0\}$ are $\mathcal{F}_{s}$-progressively measurable, for each $A \in \mathcal{B}(\mathbb{R}^p)$;
\item[(iii)] $\e_{x,i}[\int_0^t\left(\int_{\mathbb{R}^p}u^Tu\boldsymbol{\pi}_s(u)\,du+\int_{\mathbb{R}^p}v^Tv\boldsymbol{\gamma}_s(v)\,dv\right)ds]<\infty$, for each $t\geq0$;
\item[(iv)] $\liminf_{s\rightarrow\infty}\e_{x,i}\left[e^{-\rho_k s}J_k(X_s^{\boldsymbol{\pi}, \boldsymbol{\gamma}}, \alpha_s, \boldsymbol{\pi}_s, \boldsymbol{\gamma}_s)\right]=0$, for $k=1,2$;
\item[(v)] for $k=1,2$, $\e_{x,i}[\int_0^{\infty}e^{-\rho_k s}\tilde{r}_k(X_s^{\boldsymbol{\pi}, \boldsymbol{\gamma}}, \alpha_s, \boldsymbol{\pi}_s, \boldsymbol{\gamma}_s)ds]<\infty$.
\end{enumerate}
\end{assumption}
Denote by $\mathcal{U}$ (resp. $\mathcal{V}$) the set of control policies $\boldsymbol{\pi}$ (resp. $\boldsymbol{\gamma}$) satisfying Assumption \ref{assumption 1}. A policy in $\mathcal{U}$ or $\mathcal{V}$ is called an admissible control with respect to player $I$ or player $II$, respectively. The following Proposition \ref{Proposition 1} guarantees that the controlled system \eqref{2.11} is well defined and the proof is postponed to Appendix \ref{Appendix B}.

\begin{proposition}\label{Proposition 1}
For any $\boldsymbol{\pi}\in\mathcal{U}$ and $\boldsymbol{\gamma}\in\mathcal{V}$, the exploratory dynamic system \eqref{2.11} has a unique strong solution.
\end{proposition}

\begin{definition}\label{Definition 2}
The value function of player $II$ is
\begin{equation}\label{2.14}
V_2(x,i;\boldsymbol{\pi}) = \min_{\boldsymbol{\gamma}\in\mathcal{V}}J_2(x,i, \boldsymbol{\pi},\boldsymbol{\gamma}).
\end{equation}
The value function of player $I$ is
\begin{equation}\label{2.15}
V_1(x,i) = \min_{\boldsymbol{\pi}\in\mathcal{U}} J_1(x,i, \boldsymbol{\pi},\boldsymbol{\gamma}^*(\boldsymbol{\pi})),
\end{equation}
where $\boldsymbol{\gamma}^*(\boldsymbol{\pi})= \arg\min_{\boldsymbol{\gamma}\in\mathcal{V}}J_2(x,i, \boldsymbol{\pi},\boldsymbol{\gamma})$. 
\end{definition}

\begin{problem}\label{problem 1}
Design appropriate policies $\boldsymbol{\pi}^*$ and $\boldsymbol{\gamma}^*$ to satisfy the following Stackelberg equilibrium conditions, i.e.,
\begin{equation}\label{2.16}
\left\{\begin{array}{ll}
J_1(x,i, \boldsymbol{\pi}^*, \boldsymbol{\gamma}^*(\boldsymbol{\pi}^*)) \leq J_1(x,i, \boldsymbol{\pi}, \boldsymbol{\gamma}^*(\boldsymbol{\pi})), \, \forall \boldsymbol{\pi}\in\mathcal{U}; \\
J_2(x,i, \boldsymbol{\pi}, \boldsymbol{\gamma}^*(\boldsymbol{\pi})) \leq J_2(x,i, \boldsymbol{\pi}, \boldsymbol{\gamma}), \,\forall \boldsymbol{\gamma}\in\mathcal{V}.
\end{array} \right.
\end{equation}
\end{problem}
For short, we rewrite $\boldsymbol{\gamma}^*(\boldsymbol{\pi}^*)$ as $\boldsymbol{\gamma}^*$, which will not lead to any confusion in the following analysis. Then, $(\boldsymbol{\pi}^*, \boldsymbol{\gamma}^*)$ satisfying \eqref{2.16} is said to be a Stackelberg equilibrium policy pair. For short, we also rewrite $V_2(x,i;\boldsymbol{\pi}^*):=V_2(x,i)$.
\section{Equilibrium policies for ERRL LQ-SDGs}\label{Section 3}
\setcounter{equation}{0}
\renewcommand{\theequation }{3.\arabic{equation}}
In this section, we employ DPP to elucidate the Stackelberg equilibrium policy pair described in \emph{Problem \ref{problem 1}}.
First, we derive the HJBI equations for player $I$ and player $II$ (see Lemma \ref{Lemma 1}), and then present the equilibrium policies of both players as functionals of the solution to the HJBI equations (see Theorem \ref{Theorem1}, Theorem \ref{Theorem 2}), and finally prove that the identified Stackelberg equilibrium policies solve  \emph{Problem \ref{problem 1}} (see Theorem \ref{Theorem 3}). For reading ease, the proof for Lemma \ref{Lemma 1} is postponed to Appendix \ref{Appendix C}.
\begin{lemma}\label{Lemma 1}
Suppose that $\forall i\in \mathcal{M}$, $V_1(\cdot,i), V_2(\cdot,i;\boldsymbol{\pi})\in C^{2}(\mathbb{R}^n)$, then the formal HJBI equations associated with player $II$ and player $I$ are given by \eqref{3.1} and \eqref{3.2}.
\begin{equation}\label{3.1}
\begin{aligned}
\rho_2\phi_2(x,i;\boldsymbol{\pi})=&\min_{\boldsymbol{\gamma}\in\mathcal{V}}\mathcal{H}_2(x,i,\nabla\phi_2(x,i;\boldsymbol{\pi}),\Delta\phi_2(x,i;\boldsymbol{\pi}), \boldsymbol{\pi}, \boldsymbol{\gamma})\\
=&\min_{\boldsymbol{\gamma}\in\mathcal{V}}\Big[\tilde{r}_2(x,i, \boldsymbol{\pi}, \boldsymbol{\gamma})+\lambda_2\int_{\mathbb{R}^p} \boldsymbol{\gamma}(v)\ln \boldsymbol{\gamma}(v)\,dv\\
+\nabla\phi_2&(x,i;\boldsymbol{\pi})^T\tilde{b}(x,i,\boldsymbol{\pi},\boldsymbol{\gamma})+ \sum_{j=1}^m q_{ij} \phi_2(x, j; \boldsymbol{\pi})\\
+\frac{1}{2}tr&\left(\tilde{\sigma}(x,i,\boldsymbol{\pi},\boldsymbol{\gamma})\tilde{\sigma}^T(x,i,\boldsymbol{\pi},\boldsymbol{\gamma})\Delta\phi_2(x,i;\boldsymbol{\pi})\right)\Big],
\end{aligned}
\end{equation}
\begin{equation}\label{3.2}
\begin{aligned}
\rho_1\phi_1(x,i)=&\min_{\boldsymbol{\pi}\in\mathcal{U}}\mathcal{H}_1(x,i,\nabla\phi_1(x,i),\Delta\phi_1(x,i), \boldsymbol{\pi}, \boldsymbol{\gamma}^*(\boldsymbol{\pi}))\\
=&\min_{\boldsymbol{\pi}\in\mathcal{U}}\Big[\tilde{r}_1(x,i, \boldsymbol{\pi}, \boldsymbol{\gamma}^*(\boldsymbol{\pi}))+\lambda_1\int_{\mathbb{R}^p} \boldsymbol{\pi}(u)\ln \boldsymbol{\pi}(u)\,du\\
+\nabla\phi_1&(x,i)^T\tilde{b}(x,i,\boldsymbol{\pi},\boldsymbol{\gamma}^*(\boldsymbol{\pi}))+ \sum_{j=1}^m q_{ij} \phi_1(x, j)\\
+\frac{1}{2}tr&\left(\tilde{\sigma}(x,i,\boldsymbol{\pi},\boldsymbol{\gamma}^*(\boldsymbol{\pi}))\tilde{\sigma}^T(x,i,\boldsymbol{\pi},\boldsymbol{\gamma}^*(\boldsymbol{\pi}))\Delta\phi_1(x,i)\right)\Big],
\end{aligned}
\end{equation}
where 
\begin{equation*}
\boldsymbol{\gamma}^*(\boldsymbol{\pi}) = \arg \min_{\boldsymbol{\gamma}\in\mathcal{V}} \mathcal{H}_2(x,i, \nabla\phi_2(x,i;\boldsymbol{\pi}),\Delta\phi_2(x,i;\boldsymbol{\pi}), \boldsymbol{\pi}, \boldsymbol{\gamma}).
\end{equation*}
\end{lemma}
\begin{theorem}\label{Theorem1}
The equilibrium policy of player $II$ is a Gaussian distribution given by
\begin{equation}\label{3.3}
\boldsymbol{\gamma}^*(\boldsymbol{\pi})=\mathcal{N} \left(\beta, \lambda_2 \left( 2R_2(i) + D_2(i)^T \Delta\phi_2(x, i; \boldsymbol{\pi}) D_2(i) \right)^{-1}\right),
\end{equation}
where
\begin{equation}\label{3.4}
\begin{aligned}
\beta&:=-\left(2R_2(i) + D_2(i)^T\Delta\phi_2(x,i;\boldsymbol{\pi})D_2(i)\right)^{-1}\\
&\quad\left(2\theta_2(i) R_2(i)+D_2(i)^T\Delta\phi_2(x,i;\boldsymbol{\pi})D_1(i)\right)\int_{\mathbb{R}^p}u\boldsymbol{\pi}(u)\,du\\
&-\left(2R_2(i) + D_2(i)^T\Delta\phi_2(x,i;\boldsymbol{\pi})D_2(i)\right)^{-1}\\
&\quad\left(D_2(i)^T\Delta\phi_2(x,i;\boldsymbol{\pi})C(i)x+ B_2(i)^T\nabla\phi_2(x,i;\boldsymbol{\pi})\right).
\end{aligned}
\end{equation}
\end{theorem}
\begin{proof}
For any leader's policy $\boldsymbol{\pi}$, the follower aims to track optimal response under regime $i \in \mathcal{M}$ via
\[
\boldsymbol{\gamma}^*(\boldsymbol{\pi}) = \arg \min_{\boldsymbol{\gamma}\in\mathcal{V}} \mathcal{H}_2(x,i, \nabla\phi_2(x,i;\boldsymbol{\pi}),\Delta\phi_2(x,i;\boldsymbol{\pi}), \boldsymbol{\pi}, \boldsymbol{\gamma}).
\]
Subject to the constraints that  $\boldsymbol{\gamma}$ is a probability distribution, we have $\int_{\mathbb{R}^p} \boldsymbol{\gamma}(v)\,dv = 1$. Introducing the Lagrange multiplier $\mu$ to this problem yields
\begin{equation}
\begin{aligned}
&\mathcal{L}_2(x,i, \nabla\phi_2(x,i;\boldsymbol{\pi}),\Delta\phi_2(x,i;\boldsymbol{\pi}), \boldsymbol{\pi}, \boldsymbol{\gamma},\mu)\\
=&\tilde{r}_2(x,i, \boldsymbol{\pi}, \boldsymbol{\gamma})+\lambda_2\int_{\mathbb{R}^p} \boldsymbol{\gamma}(v)\ln \boldsymbol{\gamma}(v)\,dv\\
&+\nabla\phi_2(x,i;\boldsymbol{\pi})^T\tilde{b}(x,i,\boldsymbol{\pi},\boldsymbol{\gamma})+ \sum_{j=1}^m q_{ij} \phi_2(x, j; \boldsymbol{\pi})\\
&+\frac{1}{2}tr\left(\tilde{\sigma}(x,i,\boldsymbol{\pi},\boldsymbol{\gamma})\tilde{\sigma}^T(x,i,\boldsymbol{\pi},\boldsymbol{\gamma})\Delta\phi_2(x,i;\boldsymbol{\pi})\right)\\
&-\mu\left(\int_{\mathbb{R}^p}\,\boldsymbol{\gamma}(v)\,dv-1\right).
\end{aligned}
\end{equation}
Substitute the specific forms of $\tilde{b}(x,i,\boldsymbol{\pi},\boldsymbol{\gamma})$, $\tilde{\sigma}(x,i,\boldsymbol{\pi},\boldsymbol{\gamma})$, $\tilde{r}_2(x,i, \boldsymbol{\pi}, \boldsymbol{\gamma})$, then
the stationary condition is
\begin{equation*}\label{3.9}
\left\{\begin{array}{ll}
\delta \mathcal{L}_2=\int_{\mathbb{R}^p}\Big[2\theta_2(i)v^T R_2(i)\int_{\mathbb{R}^p}u\boldsymbol{\pi}(u)\,du\\
\qquad\quad+v^TR_2(i)v+ \nabla\phi_2(x,i;\boldsymbol{\pi})^TB_2(i)v\\
\qquad\quad+v^TD_2(i)^T\Delta\phi_2(x,i;\boldsymbol{\pi})C(i)x\\
\qquad\quad+\int_{\mathbb{R}^p}u^T\boldsymbol{\pi}(u)\,du\, D_1(i)^T\Delta\phi_2(x,i;\boldsymbol{\pi})D_2(i) v\\
\qquad\quad+\frac{1}{2}v^TD_2(i)^T\Delta\phi_2(x,i;\boldsymbol{\pi})D_2(i)v\\
\qquad\quad+\lambda_2(\ln \boldsymbol{\gamma}+1)-\mu\Big]\,\delta\boldsymbol{\gamma}(v)dv=0,\\
0=\int_{\mathbb{R}^p}\,\boldsymbol{\gamma}(v)\,dv-1.
\end{array} \right.
\end{equation*}
Solve for $\boldsymbol{\gamma}$, the equilibrium policy of player $II$ is
\begin{equation*}\label{3.11}
\boldsymbol{\gamma}^*(\boldsymbol{\pi})=\mathcal{N} \left(\beta, \lambda_2 \left( 2R_2(i) + D_2(i)^T \Delta\phi_2(x, i; \boldsymbol{\pi}) D_2(i) \right)^{-1}\right),
\end{equation*}
where $\beta$ is given by \eqref{3.4}.
\end{proof}
\begin{remark}By analyzing the follower's optimal response \(\boldsymbol{\gamma}^*(\boldsymbol{\pi})\), it becomes evident that the variance of \(\boldsymbol{\gamma}^*(\boldsymbol{\pi})\) increases as \(R_2(i)\) and \(D_2(i)\) decrease. This suggests that in Stackelberg games, the less influence or control the follower has over the system, the more uncertain their optimal policy tends to be. This observation explains why the disadvantaged party in such games often adopts a more aggressive strategy---a phenomenon akin to the proverb ``the barefooted fear not the shoe-wearers.''

Furthermore, in  \cite{li2019two}, the follower's optimal response \(v^*(u) = -\theta_2 u - \frac{1}{2}R_2^{-1}B^T_2 \nabla \phi_2(x; \boldsymbol{\pi})\) is a deterministic value. In our results, when $\sigma(\cdot)=0$ and $\mathcal{M}=\{1\}$ (i.e., the system is free of stochastic perturbations and regime switching), the mean of the optimal response \(\boldsymbol{\gamma}^*(\boldsymbol{\pi})\) aligns precisely with the findings of \cite{li2019two}. This consistency confirms that the optimal response is concentrated around the mean under these conditions.\end{remark}

We now turn to track player $I$'s equilibrium policy. For simplicity, let $\phi_2(x,i):=\phi_2(x,i;\boldsymbol{\pi}^*)$,
\begin{equation*}
\begin{aligned}
\Lambda_1:=&-\left(2R_2(i) + D_2(i)^T\Delta\phi_2(x,i)D_2(i)\right)^{-1}\\
&\quad\left(2\theta_2(i) R_2(i)+D_2(i)^T\Delta\phi_2(x,i)D_1(i)\right),\\
\Lambda_2:=&-\left(2R_2(i) + D_2(i)^T\Delta\phi_2(x,i)D_2(i)\right)^{-1}\\
&\quad\left(D_2(i)^T\Delta\phi_2(x,i)C(i)x+ B_2(i)^T\nabla\phi_2(x,i)\right),\\
\Lambda_3 :=& 2R_1(i) + D_1(i)^T \Delta\phi_1(x,i) D_1(i) \\
&+ D_1(i)^T \Delta\phi_1(x,i) D_2(i)\Lambda_1+ 2\theta_1(i) R_1(i)\Lambda_1 \\
&+ \Lambda_1^T D_2(i)^T \Delta\phi_1(x,i) D_2(i)\Lambda_1 \\
&+ 2\theta_1(i)^2 \Lambda_1^T R_1(i) \Lambda_1+ 2\theta_1(i) \Lambda_1^T R_1(i)\\
&+ \Lambda_1^T D_2(i)^T \Delta\phi_1(x,i) D_1(i).
\end{aligned}
\end{equation*}
and \begin{equation}\label{3.14}
\begin{aligned}
\alpha_*:=&-\Lambda_3^{-1}\Big[\Lambda_1^TB_2(i)^T\nabla\phi_1(x,i)+D_1(i)^T\Delta\phi_1(x,i)C(i)x\\
&+B_1(i)^T\nabla\phi_1(x,i)+\Lambda_1^T D_2(i)^T\Delta\phi_1(x,i)C(i)x\\
&+\left(D_1(i)^T\Delta\phi_1(x,i)D_2(i)+\Lambda_1^TD_2(i)^T\Delta\phi_1(x,i)D_2(i)\right.\\
&\left.+2\theta_1(i)R_1(i)+2\theta_1(i)^2\Lambda_1^TR_1(i)\right)\Lambda_2\Big].
\end{aligned}
\end{equation}
\begin{assumption}\label{Assumption 1}
The matrix $\Lambda_3$ is an invertible matrix.
\end{assumption}
\begin{theorem}\label{Theorem 2}
Under Assumption \ref{Assumption 1}, the equilibrium policy of player $I$ is
\begin{equation}\label{3.13}
\boldsymbol{\pi}^*=\mathcal{N} \left(\alpha_* , \lambda_1 \left(2R_1(i) + D_1(i)^T\Delta\phi_1(x,i)D_1(i)\right)^{-1}\right).
\end{equation}
\end{theorem}
\begin{proof} The proof follows a similar approach to that of Theorem \ref{Theorem1} and is therefore omitted for brevity.\end{proof}
\begin{remark}The leader's equilibrium policy $\boldsymbol{\pi}^*$ is also a Gaussian distribution, and its value depends on the current state $x$ and the regime $i$. The variance of $\boldsymbol{\pi}^*$ increases as $R_1(i)$ and $D_1(i)$ decrease. This suggests that the greater the leader's influence or control over the system, the more deterministic the equilibrium policy becomes. When the policy term weights are sufficiently large, the leader's equilibrium policy may even degenerate into a deterministic policy. This phenomenon also explains why monopolists, who possess significant ``resources'' in the field of economics, tend to adopt stable and deterministic policies.

Additionally, in \cite{li2019two}, the equilibrium policy is expressed as a definite value:
\[
u^* = \frac{\theta_1 R_2^{-1} B_2^T \nabla \phi_2(x)}{2(1 - \theta_1 \theta_2)} - \frac{R_1^{-1} B_1^T \nabla \phi_1(x)}{2(1 - \theta_1 \theta_2)^2} + \frac{\theta_2 R_1^{-1} B_1^T \nabla \phi_1(x)}{2(1 - \theta_1 \theta_2)^2}.
\]
When $\sigma(\cdot)=0$ and $\mathcal{M}=\{1\}$ (i.e., the system is free of stochastic perturbations and regime switching), the mean value of the equilibrium policy $\boldsymbol{\pi}^*$ aligns precisely with $u^*$.\end{remark}

Substitute \eqref{3.13} into \eqref{3.3}, and the equilibrium policy pair for player $I$ and player $II$ is
\begin{equation}\label{3.30}
\left\{\begin{array}{ll}
\boldsymbol{\pi}^*=\mathcal{N} \left(\alpha_* , \lambda_1 \left(2R_1(i) + D_1(i) ^T\Delta\phi_1(x,i)D_1(i) \right)^{-1}\right),\\
\boldsymbol{\gamma}^*=\mathcal{N} \left(\beta_*, \lambda_2\left(2R_2(i)  + D_2(i) ^T\Delta\phi_2(x,i)D_2(i) \right)^{-1}\right),
\end{array} \right.
\end{equation}
where $\alpha_*$ is given by \eqref{3.14} and
\begin{equation*}
\begin{aligned}
\beta_*=&-\left(2R_2(i)  + D_2(i) ^T\Delta\phi_2(x,i)D_2(i) \right)^{-1}\\
&\quad\left(2\theta_2(i)  R_2(i) +D_2(i)^T\Delta\phi_2(x,i)D_1(i) \right)\alpha_*\\
&-\left(2R_2(i)  + D_2(i) ^T\Delta\phi_2(x,i)D_2(i) \right)^{-1}\\
&\quad\left(D_2(i) ^T\Delta\phi_2(x,i)C(i) x+ B_2(i) ^T\nabla\phi_2(x,i)\right).
\end{aligned}
\end{equation*}
Then, substituting \eqref{3.30} into \eqref{3.1}, \eqref{3.2} and removing the min operator, we obtain the following coupled HJBI equations:
\begin{align}
 \rho_2\phi_2(x,i)=&\mathcal{H}_2(x, i,\nabla\phi_2(x,i),\Delta\phi_2(x,i), \boldsymbol{\pi}^*, \boldsymbol{\gamma}^*),\label{3.32}\\
 \rho_1\phi_1(x,i)=&\mathcal{H}_1(x, i,\nabla\phi_1(x,i),\Delta\phi_1(x,i), \boldsymbol{\pi}^*, \boldsymbol{\gamma}^*).\label{3.33}
\end{align}
\begin{theorem}\label{Theorem 3}
Under Assumption \ref{Assumption 1}, if $\phi_1(x,i)$ and $\phi_2(x,i)$ are solutions to the coupled HJBI equations \eqref{3.32} and \eqref{3.33}, then
\begin{enumerate}
    \item $J_1(x,i ,\boldsymbol{\pi}^*, \boldsymbol{\gamma}^*) \leq J_1(x,i, \boldsymbol{\pi}, \boldsymbol{\gamma}^*(\boldsymbol{\pi})), \, \forall \boldsymbol{\pi}\in\mathcal{U}$,
    \item $J_1(x,i, \boldsymbol{\pi}^*, \boldsymbol{\gamma}^*)=\phi_1(x,i)$;
    \item $J_2(x, i,\boldsymbol{\pi}, \boldsymbol{\gamma}^*(\boldsymbol{\pi})) \leq J_2(x,i, \boldsymbol{\pi}, \boldsymbol{\gamma}),\,\forall \boldsymbol{\gamma}\in\mathcal{V}$,
    \item $J_2(x,i, \boldsymbol{\pi}^*, \boldsymbol{\gamma}^*)=\phi_2(x,i)$.
\end{enumerate}
\end{theorem}
The proof can be found in Appendix \ref{Appendix D}.

\section{PIA for ERRL LQ-SDG}\label{Section 4}
\setcounter{equation}{0}
\renewcommand{\theequation}{4.\arabic{equation}}
Although \eqref{3.30} provides optimal feedback strategies for both players, the HJBI equations are nonlinear second-order PDEs with integral terms, making analytical solutions difficult to derive. In this section, we propose a PIA \cite{wang2020continuous} to approximate the value of the ERRL LQ-SDGs and the optimal strategies under regime switching.

The PIA shares similarities with the actor-critic framework, as it involves both a policy (actor) and a value function (critic). Specifically:
\begin{itemize}
    \item \emph{Policy (Actor)}: The strategies $\boldsymbol{\pi}$ and $\boldsymbol{\gamma}$ generate actions.
    \item \emph{Value Function (Critic)}: $J_2(x, i, \boldsymbol{\pi}, \boldsymbol{\gamma})$, $J_1(x, i, \boldsymbol{\pi}, \boldsymbol{\gamma})$ estimate state-regime values and are updated based on policy performance.
\end{itemize}
However, PIA extends beyond standard actor-critic methods by incorporating \emph{robust optimization} via the $\min$ operator over uncertainty sets $\mathcal{P}$. This key distinction makes the PIA a specialized approach tailored to stochastic control problems.

As the PIA often requires discretizing the state space, which brings about the problem of the \lq\lq curse of dimensionality\rq\rq, referring to the relevant research of \cite{jia2022policy,zhou2021actor}, we propose a critic neural networks (critic NNs) algorithm in subsection \ref{sub5.3}. This algorithm mainly parameterizes the value function and uses the least-squares (LS) method to determine the parameter weights during the policy evaluation stage.

\subsection{Convergence of PIA}\label{sub5.1}
\begin{theorem}\label{Theorem 5} (Policy improvement theorem for player $II$). Let $\boldsymbol{\pi}$ be given and $\boldsymbol{\gamma}(\boldsymbol{\pi})$ be the corresponding given admissible feedback policy. Suppose that the corresponding performance function $J_2(\cdot,i,\boldsymbol{\pi},\boldsymbol{\gamma}(\boldsymbol{\pi})) \in C^{2}(\mathbb{R}^n)$ and satisfies $\Delta J_2(x,i,\boldsymbol{\pi},\boldsymbol{\gamma}(\boldsymbol{\pi}))>0$, for any $x\in\mathbb{R}^n$ and $i \in \mathcal{M}$, the feedback policy $\tilde{\boldsymbol{\gamma}}(\boldsymbol{\pi})$ defined by
\begin{equation}\label{5.1}
\begin{aligned}
&\tilde{\boldsymbol{\gamma}}(\boldsymbol{\pi})=\\
&\mathcal{N} \left(\tilde{\beta}, \lambda_2\left(2R_2(i) + D_2(i)^T\Delta J_2(x,i,\boldsymbol{\pi},\boldsymbol{\gamma}(\boldsymbol{\pi}))D_2(i)\right)^{-1}\right),  
\end{aligned}
\end{equation}
where
\begin{equation*}
\begin{aligned}
\tilde{\beta}&=-\left(2R_2(i) + D_2(i)^T\Delta J_2(x,i,\boldsymbol{\pi},\boldsymbol{\gamma}(\boldsymbol{\pi}))D_2(i)\right)^{-1}\\
&\left(2\theta_2(i) R_2(i)+D_2(i)^T\Delta J_2(x,i,\boldsymbol{\pi},\boldsymbol{\gamma}(\boldsymbol{\pi}))D_1(i)\right)\int_{\mathbb{R}^p}u\boldsymbol{\pi}(u)\,du\\
&-\left(2R_2(i) + D_2(i)^T\Delta J_2(x,i,\boldsymbol{\pi},\boldsymbol{\gamma}(\boldsymbol{\pi}))D_2(i)\right)^{-1}\\
&\quad\left(D_2(i)^T\Delta J_2(x,i,\boldsymbol{\pi},\boldsymbol{\gamma}(\boldsymbol{\pi}))C(i)x+ B_2(i)^T\nabla J_2(x,i,\boldsymbol{\pi},\boldsymbol{\gamma}(\boldsymbol{\pi}))\right).
\end{aligned}
\end{equation*}
Then, the performance function under specific policies $\tilde{{\boldsymbol{\gamma}}}(\boldsymbol{\pi})$ and ${\boldsymbol{\gamma}}(\boldsymbol{\pi})$ satisfies
\begin{equation}\label{5.3}
J_2(x,i,\boldsymbol{\pi},\tilde{\boldsymbol{\gamma}}(\boldsymbol{\pi}))\leq J_2(x,i,\boldsymbol{\pi},\boldsymbol{\gamma}(\boldsymbol{\pi})),\, \forall x\in\mathbb{R}^n, i \in \mathcal{M}.
\end{equation}
\end{theorem}
The proof can be found in Appendix \ref{Appendix E}.

\begin{theorem}\label{Theorem 6} (Policy improvement theorem for player $I$). Let $\boldsymbol{\gamma}$ be given and $\boldsymbol{\pi}$ be the corresponding given admissible feedback policy. Suppose that the corresponding performance function $J_1(\cdot,i,\boldsymbol{\pi},\boldsymbol{\gamma}) \in C^{2}(\mathbb{R}^n)$ and satisfies $\Delta J_1(x,i,\boldsymbol{\pi},\boldsymbol{\gamma})>0$, for any $x\in\mathbb{R}^n$ and $i \in \mathcal{M}$, the feedback policy $\tilde{\boldsymbol{\pi}}$ defined by
\begin{equation}
\begin{aligned}
&\tilde{\boldsymbol{\pi}}=\\
&\mathcal{N} \left(\tilde{\alpha}, \lambda_1 \left(2R_1(i) + D_1(i)^T\Delta J_1(x,i,\boldsymbol{\pi},\boldsymbol{\gamma})D_1(i)\right)^{-1}\right),
\end{aligned}
\end{equation}
where
\begin{equation*}
\begin{aligned}
\tilde{\alpha}:=&-\dot{\Lambda}_3^{-1}\Big[B_1(i)^T\nabla J_1(x,i,\boldsymbol{\pi},\boldsymbol{\gamma})\\
&\quad+\dot{\Lambda}_1^TB_2(i)^T\nabla J_1(x,i,\boldsymbol{\pi},\boldsymbol{\gamma})\\
&\quad+D_1(i)^T\Delta J_1(x,i,\boldsymbol{\pi},\boldsymbol{\gamma})C(i)x\\
&\quad+\dot{\Lambda}_1^T D_2(i)^T\Delta J_1(x,i,\boldsymbol{\pi},\boldsymbol{\gamma})C(i)x\\
&\quad+\left(D_1(i)^T\Delta J_1(x,i,\boldsymbol{\pi},\boldsymbol{\gamma})D_2(i)\right.\\
&\quad+\dot{\Lambda}_1^TD_2(i)^T\Delta J_1(x,i,\boldsymbol{\pi},\boldsymbol{\gamma})D_2(i)\\
&\quad\left.+2\theta_1(i)R_1(i)+2\theta_1(i)^2\dot{\Lambda}_1^TR_1(i)\right)\dot{\Lambda}_2\Big],
\end{aligned}
\end{equation*}
\begin{equation*}
\begin{aligned}
\dot{\Lambda}_1:=&-\left(2R_2(i) + D_2(i)^T\Delta J_2(x,i,\boldsymbol{\pi},\boldsymbol{\gamma})D_2(i)\right)^{-1}\\
&\left(2\theta_2(i) R_2(i)+D_2(i)^T\Delta J_2(x,i,\boldsymbol{\pi},\boldsymbol{\gamma})D_1(i)\right),\\
\dot{\Lambda}_2:=&-\left(2R_2(i) + D_2(i)^T\Delta J_2(x,i,\boldsymbol{\pi},\boldsymbol{\gamma})D_2(i)\right)^{-1}\\
&\left(D_2(i)^T\Delta J_2(x,i,\boldsymbol{\pi},\boldsymbol{\gamma})C(i)x+ B_2(i)^T\nabla J_2(x,i,\boldsymbol{\pi},\boldsymbol{\gamma})\right),\\
\dot{\Lambda}_3:=&2R_1(i)+D_1(i)^T\Delta J_1(x,i,\boldsymbol{\pi},\boldsymbol{\gamma})D_1(i)\\
&+D_1(i)^T\Delta J_1(x,i,\boldsymbol{\pi},\boldsymbol{\gamma})D_2(i)\dot{\Lambda}_1+2\theta_1(i)\dot{\Lambda}_1^TR_1(i)\\
&+\dot{\Lambda}_1^TD_2(i)^T\Delta J_1(x,i,\boldsymbol{\pi},\boldsymbol{\gamma})D_2(i)\dot{\Lambda}_1+2\theta_1(i)R_1(i)\dot{\Lambda}_1\\
&+2\theta_1(i)^2\dot{\Lambda}_1^TR_1(i)\dot{\Lambda}_1 +\dot{\Lambda}_1^TD_2(i)^T\Delta J_1(x,i,\boldsymbol{\pi},\boldsymbol{\gamma})D_1(i).
\end{aligned}
\end{equation*}
Then the performance function under specific policies $\tilde{{\boldsymbol{\pi}}}$ and ${\boldsymbol{\pi}}$ satisfies
\begin{equation}
J_1(x,i,\tilde{\boldsymbol{\pi}},\boldsymbol{\gamma})\leq J_1(x,i,\boldsymbol{\pi},\boldsymbol{\gamma}),\, \forall x\in\mathbb{R}^n, i \in \mathcal{M}.
\end{equation}
\end{theorem}
\begin{proof}
The proof follows a similar procedure to that of Theorem \ref{Theorem 5} and is therefore omitted for brevity.
\end{proof}
\subsection{Algorithm design}\label{sub5.2}
We have developed a PIA in the game setting with regime switching. The main algorithm consists of two parallel steps: policy improvement and policy evaluation. For policy improvement, Theorem \ref{Theorem 5} and Theorem \ref{Theorem 6} ensure convergence. For policy evaluation, we compute the value functions of player $I$ and player $II$ respectively under any admissible policy pair $(\boldsymbol{\pi},\boldsymbol{\gamma})$ for all regimes $i \in \mathcal{M}$. The algorithm mainly consists of four steps:

\emph{Step 1 (Initial value):} For all $i \in \mathcal{M}$, initialize $V_1^{(0)}(x,i)$, $V_2^{(0)}(x,i)$, $\alpha^{(0)}$, $\beta^{(0)}$. Then, the initial policy pair $(\boldsymbol{\pi}^{(0)},\boldsymbol{\gamma}^{(0)})$ is
\begin{equation*}
\begin{aligned}
\mathcal{N} \left(\alpha^{(0)},\lambda_1\left(2R_1(i) + D_1(i)^T\Delta V^{(0)}_1(x,i)D_1(i)\right)^{-1}\right),\\
\mathcal{N} \left(\beta^{(0)},\lambda_2\left(2R_2(i) + D_2(i)^T\Delta V^{(0)}_2(x,i)D_2(i)\right)^{-1}\right).
\end{aligned}
\end{equation*}
Let $s = 0$ and $\epsilon$ be a small constant.

\emph{Step 2 (Iteration for the leader): } For each $i \in \mathcal{M}$,
\begin{itemize}
\item Value function update step:
\begin{equation*}
\begin{aligned}
&V_1^{(s+1)}(x, i) \\
&=\frac{1}{\rho_1}\mathcal{H}_1(x, i, \nabla V^{(s)}_1(x, i),\Delta V^{(s)}_1(x, i), \boldsymbol{\pi}^{(s)}, \boldsymbol{\gamma}^{(s)}).
\end{aligned}
\end{equation*}

\item Policy update step:
\begin{equation*}
\begin{aligned}
&\boldsymbol{\pi}^{(s+1)}=\\
&\mathcal{N} \left(\alpha^{(s+1)},\lambda_1\left(2R_1(i) + D_1(i)^T\Delta V^{(s+1)}_1(x,i)D_1(i)\right)^{-1}\right),
\end{aligned}
\end{equation*}
where $\alpha^{(s+1)}$ is updated exactly as derived in Section \ref{Section 3}, using the known matrices $B_1(i),B_2(i),D_1(i), D_2(i)$ and the value function $\nabla V_1^{(s+1)}(x,i)$, $\Delta V_1^{(s+1)}(x,i)$, $\nabla V_2^{(s)}(x,i),\Delta V_2^{(s)}(x,i)$.
\end{itemize}

\emph{Step 3 (Iteration for the follower):} For each $i \in \mathcal{M}$,
\begin{itemize}
\item Value function update step:
\begin{equation*}\label{5.19}
\begin{aligned}
&V_2^{(s+1)}(x, i)\\
&=\frac{1}{\rho_2}\mathcal{H}_2(x, i, \nabla V^{(s)}_2(x, i),\Delta V^{(s)}_2(x, i), \boldsymbol{\pi}^{(s+1)}, \boldsymbol{\gamma}^{(s)}).
\end{aligned}
\end{equation*}

\item Policy update step:
\begin{equation*}\label{5.20}
\begin{aligned}
&\boldsymbol{\gamma}^{(s+1)}=\\
&\mathcal{N} \left(\beta^{(s+1)},\lambda_2\left(2R_2(i) + D_2(i)^T\Delta V^{(s+1)}_2(x,i)D_2(i)\right)^{-1}\right),
\end{aligned}
\end{equation*}
where $\beta^{(s+1)}$ is updated according to the gradients of $V_1^{(s+1)}(x,i)$, $V_2^{(s+1)}(x,i)$ as previously established.
\end{itemize}

\emph{Step 4 (Convergence criterion):} If $\max_{i \in \mathcal{M}} |V_k^{(s + 1)}(x,i) - V_k^{(s)}(x,i)|\leq \epsilon$ for $k= 1,2$, stop and output
\begin{equation*}
\left\{\begin{array}{ll}
\boldsymbol{\pi}^*=\boldsymbol{\pi}^{(s+1)},\,
\boldsymbol{\gamma}^*=\boldsymbol{\gamma}^{(s+1)},\\
V_1(x,i)=V_1^{(s + 1)}(x,i),\\
V_2(x,i)=V_2^{(s + 1)}(x,i).
\end{array} \right.
\end{equation*}
Otherwise, let $s = s + 1$ and go back to execute \emph{Steps 2-3} in sequence.

We summarize the results into \emph{Algorithm \ref{Algorithm 1}}.
\begin{algorithm}[h!]
\caption{Policy improvement algorithm}
\label{Algorithm 1}
\begin{algorithmic}
\STATE \textbf{Step 1 (Initial value):} For all $i \in \mathcal{M}$, initialize $V_1^{(0)}(x,i)$, $V_2^{(0)}(x,i)$, $\alpha^{(0)}$, $\beta^{(0)}$. Then get the initial policy pair $(\boldsymbol{\pi}^{(0)},\boldsymbol{\gamma}^{(0)})$.
Let $s = 0$ and $\epsilon$ be a small constant.
\STATE \textbf{Step 2 (Iteration for the leader): } For $i \in \mathcal{M}$:
\begin{itemize}
\item Update value function $V_1^{(s+1)}(x,i)$.
\item Update $\Lambda_1^{(s)}(i)$, $\Lambda_2^{(s)}(i)$, $\Lambda_3^{(s)}(i)$, $\alpha^{(s+1)}$ to obtain policy $\boldsymbol{\pi}^{(s+1)}$.
\end{itemize}
\STATE \textbf{Step 3 (Iteration for the follower):} For $i \in \mathcal{M}$:
\begin{itemize}
\item Update value function $V_2^{(s+1)}(x,i)$.
\item Update $\beta^{(s+1)}$ to obtain policy $\boldsymbol{\gamma}^{(s+1)}$.
\end{itemize}
\STATE \textbf{Step 4 :} If $\max_{i \in \mathcal{M}} |V_k^{(s + 1)}(x,i) - V_k^{(s)}(x,i)| \leq \epsilon$ for $k = 1,2$, stop and output the optimal policies and value functions.
Otherwise, let $s = s + 1$ and go back to execute \textbf{Steps 2-3} in sequence.
\end{algorithmic}
\end{algorithm}

\subsection{Implement the algorithm using critic NNs}\label{sub5.3}
The implementation process of \emph{Algorithm \ref{Algorithm 1}} faces significant challenges in practical applications. Traditional numerical solution methods typically require discretizing the continuous state space, a process highly susceptible to the \lq\lq curse of dimensionality\rq\rq---specifically, as the dimensionality of the system state increases, the required computational effort grows exponentially. To better address the challenges of high-dimensionality, a promising direction is to regard artificial neural networks as a more flexible and efficient function approximation tool \cite{zhou2021actor, jia2022policy}. In this subsection, we propose a critic NNs algorithm to implement the  policy evaluation stage using the least-squares (LS) method.

Since the regime $i \in \mathcal{M} = \{1, \dots, m\}$ takes discrete values, we can approximate the value function for each regime using separate sets of weights. According to the Weierstrass high-order approximation theorem, the smooth value function $V_k(x, i)$ for $k=1,2$ can be approximated using the critic NNs as follows
\begin{equation}\label{5.22}
\overline{V}_{k}(x, i) =(W_{k}(i))^T \varphi_{k}(x),
\end{equation}
where $\varphi_k(x) \in \mathbb{R}^{h_k}$ is the critic NNs activation function vector with $h_{k}$ being the number of neurons in the critic NNs hidden layer, and $W_{k}(i)\in\mathbb{R}^{h_{k}}$ is a constant weight vector corresponding to regime $i$. The corresponding gradient and Hessian matrix with respect to state $x$ can be obtained as
\begin{equation*}
\nabla\overline{V}_{k}(x, i) =(W_{k}(i))^T\frac{\partial\varphi_{k}(x)}{\partial x},
\end{equation*}
\begin{equation*}
\Delta\overline{V}_{k}(x, i) =(W_{k}(i))^T\frac{\partial^2\varphi_{k}(x)}{\partial x^2},
\end{equation*}
where ${\partial\varphi_{k}(x)}/{\partial x}$ is a two-dimensional tensor of $h_{k}\times n$, and ${\partial^2\varphi_{k}(x)}/{\partial x^2}$ is a three-dimensional tensor of $h_{k}\times n\times n$, for $k=1,2$.

Substitute the approximate form $\overline{V}_{k}(x,i)$ of $V_{k}(x,i)$ in equation \eqref{5.22} into \emph{Algorithm \ref{Algorithm 1}}, so that $\Lambda_1^{(s)}$, $\Lambda_2^{(s)}$, $\Lambda_3^{(s)}$, $\alpha^{(s+1)}$, $\boldsymbol{\pi}^{(s + 1)}$ and $\beta^{(s+1)}$, $\boldsymbol{\gamma}^{(s + 1)}$ in the \emph{Policy update step} can be updated. Subsequently, the \emph{Value function update step} becomes a key point of concern. For each regime $i \in \mathcal{M}$, the corresponding residual errors for player $I$ and player $II$ are
\begin{equation*}
\begin{aligned}
\tilde{e}_1^{(s)}(i)&=(W_{1}^{(s+1)}(i))^T \varphi_{1}(x)\\
&-\frac{1}{\rho_1}\mathcal{H}_1(x, i, \nabla\overline{V}^{(s)}_1(x,i),\Delta\overline{V}^{(s)}_1(x,i),\boldsymbol{\pi}^{(s)},\boldsymbol{\gamma}^{(s)}),\\
\tilde{e}_2^{(s)}(i)&=(W_{2}^{(s+1)}(i))^T \varphi_{2}(x)\\
&-\frac{1}{\rho_2}\mathcal{H}_2(x, i, \nabla\overline{V}^{(s)}_2(x,i),\Delta\overline{V}^{(s)}_2(x,i),\boldsymbol{\pi}^{(s+1)},\boldsymbol{\gamma}^{(s)}).
\end{aligned}
\end{equation*}

Now our objective is to find the appropriate weight vector $W_{k}^{(s + 1)}(i)$ such that the residual error $\tilde{e}_k^{(s)}(i)$ approaches zero. In fact, we can regard this problem as a LS problem, where $W_{k}^{(s + 1)}(i)$ is the unknown parameter, $\varphi_{k}(x)$ is the regression vector, and $\frac{1}{\rho_k}\mathcal{H}_k$ is a known parameter calculated using the observed system data. After collecting the system data at $M$ points ($M > h_{k}$), the update law of the critic NNs for each regime $i$ is
\begin{equation}\label{5.25}
W_k^{(s+1)}(i) = \left(\xi_k^{(s)}(\xi_k^{(s)})^T \right)^{-1}\xi_k^{(s)} \chi_k^{(s)}(i),
\end{equation}
where $\xi_k^{(s)}=\begin{bmatrix}\varphi_k(x_{t_1})&\cdots&\varphi_k(x_{t_M})\end{bmatrix}\in \mathbb{R}^{h_k \times M}$ and $\chi_k^{(s)}(i) =\frac{1}{\rho_k}\begin{bmatrix}\mathcal{H}_k(t_1, i)&\cdots&\mathcal{H}_k(t_M, i)\end{bmatrix}^T\in \mathbb{R}^{M \times 1}$, for $k=1,2$.

We summarize the results into \emph{Algorithm \ref{Algorithm 2}}.
\begin{algorithm}[h!]
\caption{Critic neural networks algorithm}
\label{Algorithm 2}
\begin{algorithmic}
\STATE \textbf{Step 1 (Initial value):} For all $i \in \mathcal{M}$, initialize $W_{1}^{(0)}(i)$, $W_{2}^{(0)}(i)$, $\alpha^{(0)}$, $\beta^{(0)}$. Collect system data at $M$ points. Select appropriate activation functions $\varphi_1$ and $\varphi_2$. Then, $\Delta\overline{V}_{k}(x,i) =(W_{k}(i))^T\frac{\partial^2\varphi_{k}(x)}{\partial x^2}$, for $k=1,2$. The initial policy pair is
\begin{equation*}
\begin{aligned}
\boldsymbol{\pi}^{(0)}
=\mathcal{N} \left(\alpha^{(0)},\lambda_1\left(2R_1(i) + D_1(i)^T\Delta\overline{V}_{1}(x,i)D_1(i)\right)^{-1}\right),\\
\boldsymbol{\gamma}^{(0)}
=\mathcal{N} \left(\beta^{(0)},\lambda_2\left(2R_2(i) + D_2(i)^T\Delta\overline{V}_{2}(x,i)D_2(i)\right)^{-1}\right).
\end{aligned}
\end{equation*}
Let $s = 0$ and $\epsilon$ be a small constant.
\STATE \textbf{Step 2 (Iteration for the leader):} For $i \in \mathcal{M}$:
\begin{itemize}
\item Update the weight $W_{1}^{(s+1)}(i)$ according to \eqref{5.25}.
\item Update $\Lambda_1^{(s)}$, $\Lambda_2^{(s)}$, $\Lambda_3^{(s)}$, $\alpha^{(s+1)}$ to obtain policy $\boldsymbol{\pi}^{(s+1)}$.
\end{itemize}
\STATE \textbf{Step 3 (Iteration for the follower):} For $i \in \mathcal{M}$:
\begin{itemize}
\item Update the weight $W_{2}^{(s+1)}(i)$ according to \eqref{5.25}.
\item Update $\beta^{(s+1)}$ to obtain policy $\boldsymbol{\gamma}^{(s+1)}$.
\end{itemize}
\STATE \textbf{Step 4 (Convergence criterion):} If $\max_{i \in \mathcal{M}} ||W_{k}^{(s+1)}(i)-W_{k}^{(s)}(i)|| \leq \epsilon$ for $k = 1,2$, stop and output
\begin{equation*}
\left\{\begin{array}{ll}
\boldsymbol{\pi}^*=\boldsymbol{\pi}^{(s+1)},\,
\boldsymbol{\gamma}^*=\boldsymbol{\gamma}^{(s+1)},\\
V_1(x,i)=(W_{1}^{(s+1)}(i))^T \varphi_{1}(x),\\
V_2(x,i)=(W_{2}^{(s+1)}(i))^T \varphi_{2}(x).
\end{array} \right.
\end{equation*}
Otherwise, let $s = s + 1$ and go back to execute \textbf{Steps 2-3} in sequence.
\end{algorithmic}
\end{algorithm}

\section{Simulation results}\label{Section 5}
\setcounter{equation}{0}
\renewcommand{\theequation}{5.\arabic{equation}}
In this section, we provide a simulation example to illustrate the effectiveness of the proposed critic NNs algorithm, and use this algorithm to analyze some characteristics of the two-player Stackelberg game under regime switching.

Suppose the Markov chain has two regimes, i.e., $\mathcal{M}=\{1,2\}$. The parameters in the regime-switching SDE are selected as follows:
\begin{equation*}
B_1(1) = \begin{bmatrix}
0.3 \\
0.8
\end{bmatrix}, \,
B_2(1) = \begin{bmatrix}
0.8 \\
0.2
\end{bmatrix}, 
D_1(1) = \begin{bmatrix}
0.3 \\
0.8
\end{bmatrix}, \,
\end{equation*}
\begin{equation*}
D_2(1) = \begin{bmatrix}
-0.8 \\
0.2
\end{bmatrix},\,
A(1) = \begin{bmatrix}
0 & 1 \\
0 & -1
\end{bmatrix}, \,
C(1) = \begin{bmatrix}
1 & 0 \\
0 & -1
\end{bmatrix};
\end{equation*}

\begin{equation*}
B_1(2) = \begin{bmatrix}
0.4 \\
0.7
\end{bmatrix}, \,
B_2(2) = \begin{bmatrix}
0.6 \\
0.3
\end{bmatrix}, \,
D_1(2) = \begin{bmatrix}
0.4 \\
0.6
\end{bmatrix}, 
\end{equation*}
\begin{equation*}
D_2(2) = \begin{bmatrix}
-0.6 \\
0.4
\end{bmatrix},\,
A(2) = \begin{bmatrix}
0 & 0.8 \\
0 & -1.2
\end{bmatrix}, \,
C(2) = \begin{bmatrix}
0.8 & 0 \\
0 & -0.8
\end{bmatrix}.
\end{equation*}

The parameters in the performance functions for each regime are selected as follows:
\begin{equation*}
Q_1(1) = \begin{bmatrix}
0.1 & 0 \\
0 & 0.1
\end{bmatrix}, \,
Q_2(1) = \begin{bmatrix}
0.5 & 0 \\
0 & 0.5
\end{bmatrix},
\end{equation*}
\begin{equation*}
R_1(1) = 5, \, R_2(1) = 8, \, \theta_1(1) = 0.4, \, \theta_2(1) = 0.1;
\end{equation*}

\begin{equation*}
Q_1(2) = \begin{bmatrix}
0.2 & 0 \\
0 & 0.2
\end{bmatrix}, \,
Q_2(2) = \begin{bmatrix}
0.4 & 0 \\
0 & 0.4
\end{bmatrix},
\end{equation*}
\begin{equation*}
R_1(2) = 6, \, R_2(2) = 7, \, \theta_1(2) = 0.3, \, \theta_2(2) = 0.2.
\end{equation*}

The generator matrix of the Markov chain is
\begin{equation*}
\Gamma = \begin{bmatrix}
-0.5 & 0.5 \\
0.4 & -0.4
\end{bmatrix}.
\end{equation*}

Other parameters are $\rho_1 = 3, \rho_2 = 3, \lambda_1 = 3, \lambda_2 = 4$. The initial state is $x=\begin{bmatrix}3 & 3\end{bmatrix}^T$, and the initial regime is $\alpha_0 = 1$. The algorithm starts from $W_{k}^{(0)}(i)=\begin{bmatrix}1 & 1 & 1 & 1\end{bmatrix}^T$, $\alpha^{(0)}=0$, and $\beta^{(0)}=0$ for $k \in \{1, 2\}$ and $i \in \{1, 2\}$. Set the convergence criterion as $\epsilon=10^{-3}$. The activation function is selected as
\begin{equation}\label{6.1}
\varphi_1(x)=\varphi_2(x)=\begin{bmatrix}x_1^2 & x_1x_2 & x_2^2 & 1\end{bmatrix}^T.
\end{equation}

\subsection{The convergence of the algorithm}
We first implement the critic NNs \emph{Algorithm \ref{Algorithm 2}}. The weights of the critic NNs are updated after collecting $M = 50$ points. Each point is sampled every $0.01$ seconds, and the $6$th point is selected as the observation point. Fig. \ref{fig:img1} illustrates the convergence of the weight vectors in the critic NNs for both the leader and the follower under different regimes. As depicted, both networks successfully meet the convergence criteria after $81$ iterations.
\begin{figure}[h!]
	\centering
	\includegraphics[width=3.2in]{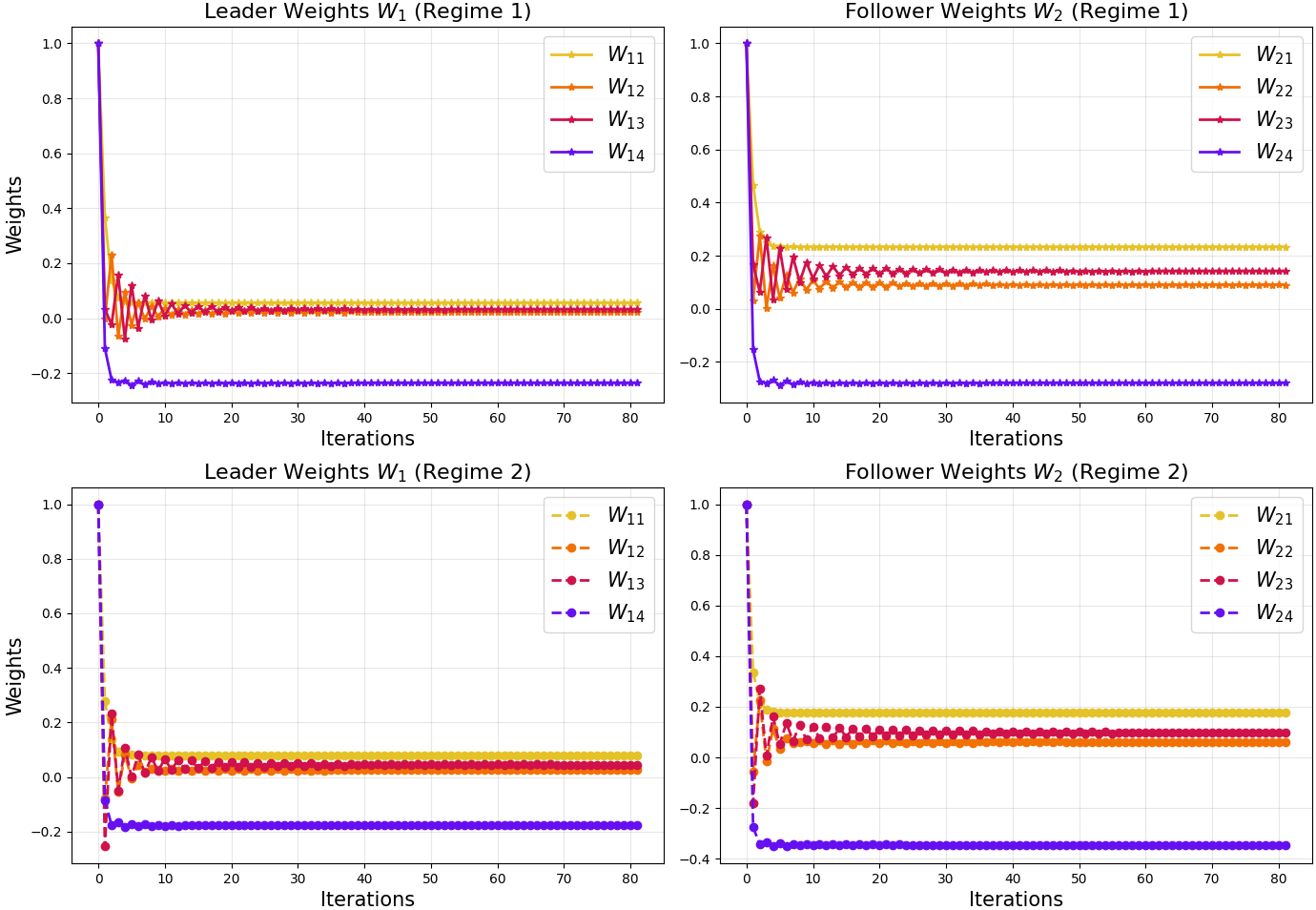}
	\caption{Convergence of the weight vectors.}
	\label{fig:img1}
\end{figure}

The loss function is crucial for evaluating the performance of the algorithm. We define the loss function using the Kullback-Leibler (KL) divergence:
\begin{equation*}
\begin{aligned}
L(\mu, \Sigma) =& \frac{1}{2} \Big(tr(\Sigma^{-1} \Sigma^*) + (\mu - \mu^*)^T \Sigma^{-1} (\mu - \mu^*)\\
&- 3 + \ln \frac{\det(\Sigma^*)}{\det(\Sigma)} \Big),
\end{aligned}
\end{equation*}
where $\mu$ and $\Sigma$ are the mean and covariance matrix of the current distribution, and $\mu^*$ and $\Sigma^*$ are the mean and covariance matrices of the target distribution. As shown in
Fig. \ref{fig:img2}, the value of the loss function drops below $10^{-7}$ after a finite number of iterations for both regimes, indicating stable convergence.
\begin{figure}[h!]
	\centering
	\includegraphics[width=3.2in]{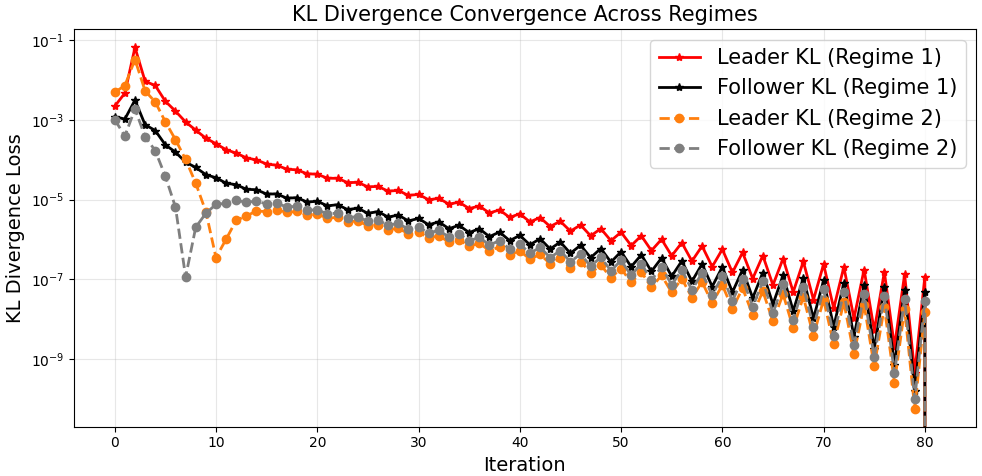}
	\caption{Convergence of the KL loss function.}
	\label{fig:img2}
\end{figure}

\subsection{Impact of temperature parameter on value function}
Figure \ref{fig:img3} illustrates the regime-dependent value functions $V_k(x, i)$ for both the leader ($k=1$) and the follower ($k=2$) under various temperature parameter configurations, specifically $\lambda_1, \lambda_2 \in \{(1, 1), (4, 4), (6, 6)\}$.

A primary observation from the numerical results is that the amplitude of the value functions exhibits a monotonic decrease as the temperature parameter $\lambda$ increases. This phenomenon is consistent across both regimes. Within the proposed framework, where negative entropy is incorporated as an auxiliary reward to promote exploration, a larger $\lambda$ effectively heightens the agents' ability to traverse the state space and adapt to environmental transitions. This enhanced exploration facilitates the discovery of more globally optimal policies, which manifests as a reduction in the magnitude of the value surfaces.
Furthermore, a comparative analysis of the two agents reveals a distinct disparity in sensitivity to the temperature parameter:
\begin{itemize}
    \item For the leader, the value function $V_1(x, i)$ shows significant vertical displacement as $\lambda$ varies, indicating that the leader's strategic value is highly sensitive to the exploration-exploitation trade-off.
    \item For the follower, the value function $V_2(x, i)$ exhibits much tighter clustering across the different $\lambda$ settings.
\end{itemize}
This suggests that the follower's value profile is relatively robust to changes in the temperature parameter, potentially due to its reactive role within the Stackelberg structure.
\begin{figure}[h!]
	\centering
	\includegraphics[width=3.4in]{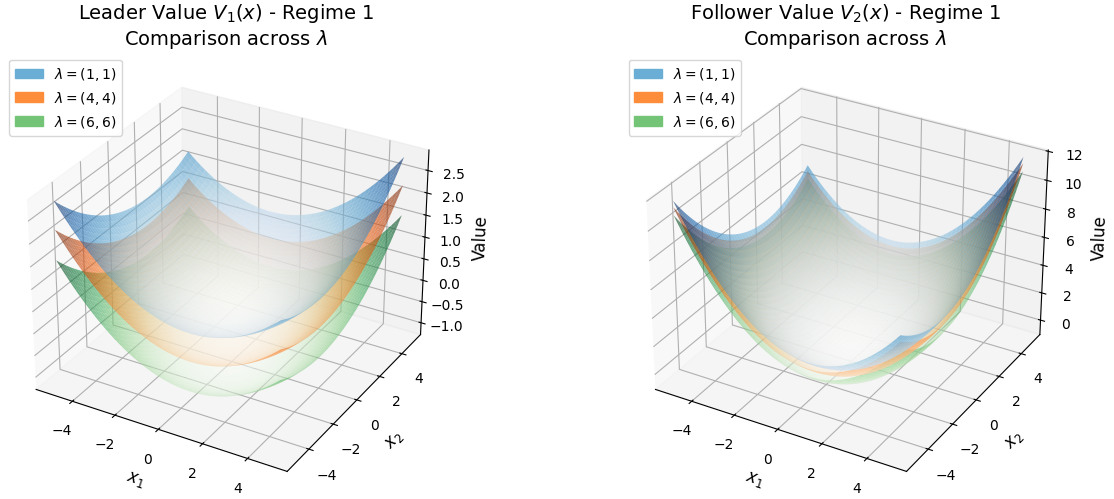}
	\caption{Evolution of the leader and follower value surfaces under different temperature parameter $\lambda$.}
	\label{fig:img3}
\end{figure}

Fig. \ref{fig:img4} plots the evolution of the leader’s value surface \(V_1(x)\) in Regime $1$ while fixing \(\lambda_2=4\). The results show that the magnitude of \(V_1(x)\) decreases monotonically with increasing \(\lambda_1\). As a key hyperparameter controlling exploration intensity, a larger \(\lambda_1\) amplifies the negative entropy reward, enabling the leader to explore the full state space instead of local high-reward regions. This leads to the discovery of globally robust policies, evidenced by the downward shift in the value surface—though reducing the overall expected value, the strategy gains superior adaptability to dynamic environments.
\begin{figure}[h!]
	\centering
	\includegraphics[width=3.2in]{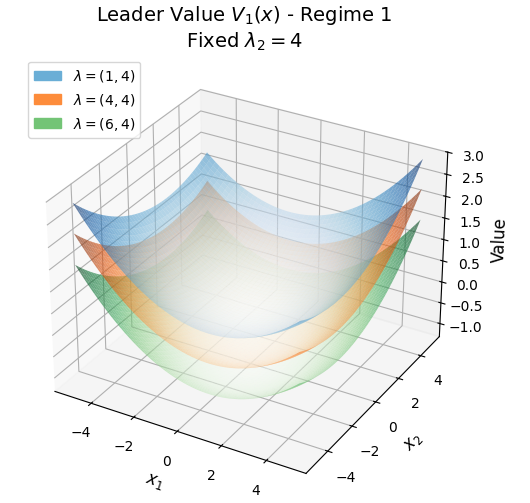}
	\caption{Evolution of the leader's value surface under different temperature parameters $\lambda$ with the follower's temperature parameter fixed in regime $1$.}
	\label{fig:img4}
\end{figure}

Figure \ref{fig:img5} illustrates a numerical comparison of the leader and follower value functions, $V_1(x,i)$ and $V_2(x,i)$, under two distinct regimes (regime 1 and regime 2) with the parameter $\lambda = (4, 4)$. In terms of surface morphology, both value functions exhibit typical convex quadratic characteristics. Numerical analysis reveals that the two regime mechanisms exert starkly opposing effects on the players. For the leader, the value function surface of regime 2 is uniformly higher than that of regime 1, implying that the leader incurs a higher cost under regime 2. Conversely, the follower's value function under regime 1 is significantly higher than under regime 2. Furthermore, its numerical magnitude (with a peak approaching $10$) is much larger than that of the leader (with a peak around $3.5$), indicating that the follower is far more sensitive to fluctuations in the state variables. This "cross-domination" phenomenon reveals the asymmetric impact of regime switching on the distribution of interests between the players, suggesting a fundamental drift in the structure of the game equilibrium across different states.
\begin{figure}[h!]
	\centering
	\includegraphics[width=3.4in]{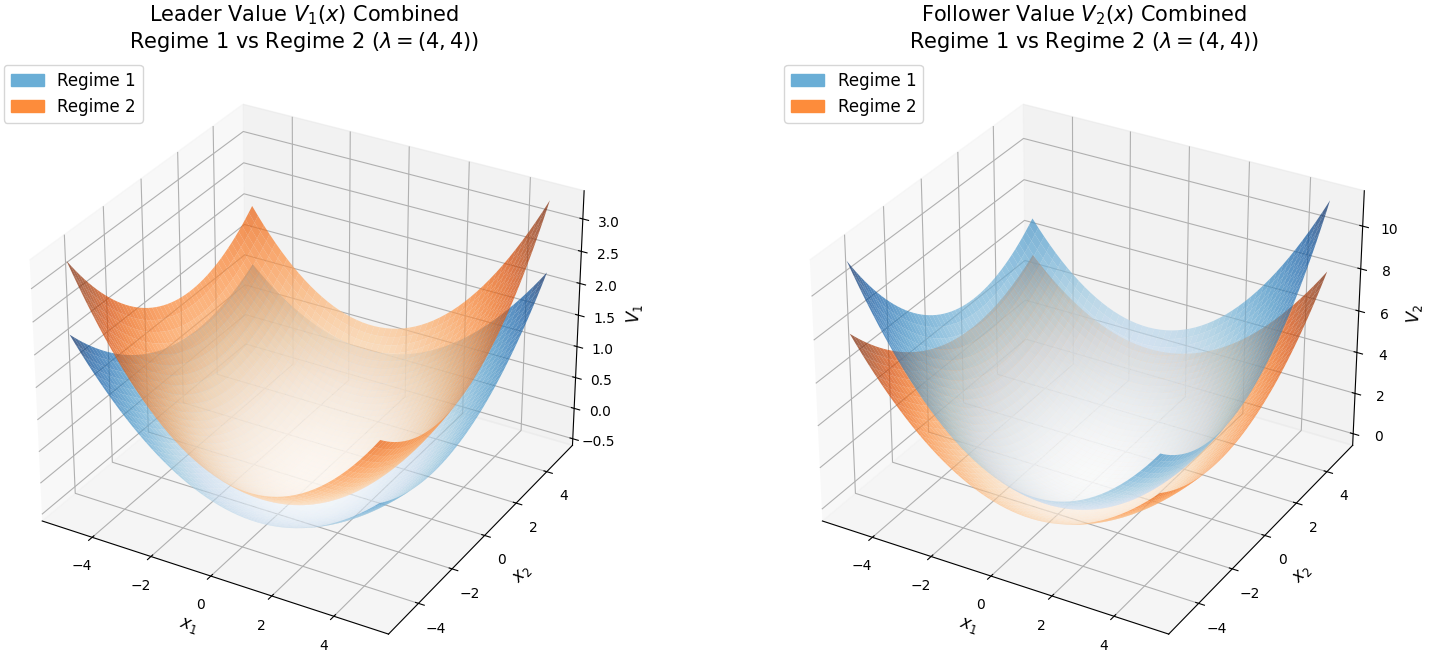}
	\caption{omparison of leader and follower value functions $V_1(x)$ and $V_2(x)$ across Regime 1 and Regime 2 with $\lambda = (4, 4)$.}
	\label{fig:img5}
\end{figure}

\subsection{Impact of temperature parameter on policy pair of players}
As shown in Fig. \ref{fig:img6}, we examine the equilibrium policy pairs under the following temperature parameter combinations: $\lambda_1,\lambda_2 \in \{(1, 1), (4, 4), (6, 6)\}$. The experimental results reveal the properties of the regime-dependent equilibrium policies for a given active regime $i \in \mathcal{M}$:
\begin{itemize}
\item \textbf{Mean Stability}: The mean of the optimal policy is highly robust to changes in temperature parameters within a specific regime, meaning that exploration intensity has little impact on the policy's expected behavior, which is primarily dictated by the system's current structural state $i$.
\item \textbf{Variance Sensitivity}: As the temperature parameter increases, the variance of the policy increases significantly, which reflects the exploration-exploitation trade-off in reinforcement learning---a higher temperature parameter promotes exploration by increasing policy entropy, manifested as a wider dispersion of the distribution to better anticipate and react to potential regime shifts.
\end{itemize}
\begin{figure}[h!]
	\centering
	\includegraphics[width=3.4in]{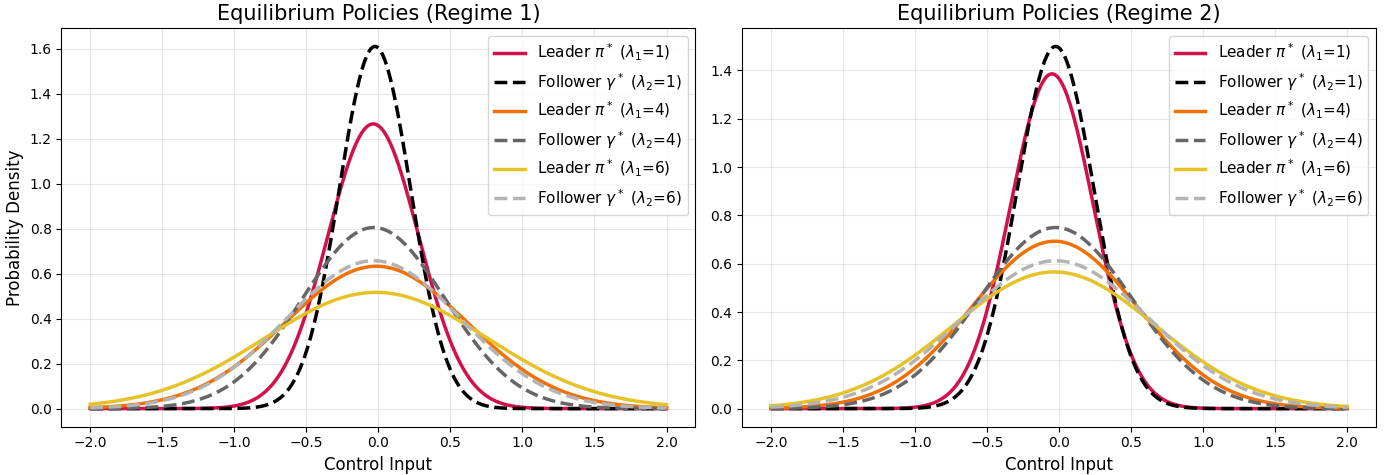}
	\caption{Equilibrium policies under varying temperature parameters.}
	\label{fig:img6}
\end{figure}

\subsection{Escaping Local Optima: An Illustrative Case Study}
This subsection uses an example to reveal the limitation of classical models, which are prone to falling into local optima. In contrast, the PIA algorithm adopted under our exploratory framework can effectively escape local optima and achieve the search for global optimal solutions.
For dynamic system \eqref{2.1} and performance function
\begin{equation}
J^{cl}_k(x, u, v) = \int_0^{\infty}e^{-\rho_k s} r_k(x_s, u_s, v_s) \, ds, \, k = 1, 2.
\end{equation}
Referring to the similar steps in Section \ref{Section 3}, we can obtain the optimal policy pair is
\begin{equation}\label{4.2}
\left\{\begin{array}{ll}
u^* = \frac{\theta_1 R_2^{-1} B_2^T \nabla V^{cl}_2(x)}{2(1 - \theta_1 \theta_2)} - \frac{R_1^{-1} B_1^T \nabla V^{cl}_1(x)}{2(1 - \theta_1 \theta_2)^2} + \frac{\theta_2 R_1^{-1} B_1^T \nabla V^{cl}_1(x)}{2(1 - \theta_1 \theta_2)^2},\\
v^* = -\theta_2 u^* - \frac{1}{2}R_2^{-1}B^T_2 \nabla V^{cl}_2(x),
\end{array} \right.
\end{equation}
and the corresponding coupled HJB equations are
\begin{equation}\label{4.3}
\begin{aligned}
r_2(x, u^*, v^*)+\nabla V^{cl}_2(x)^T(Ax + B_1u^* + B_2v^*)=\rho_2 V^{cl}_2(x),\\
r_1(x, u^*, v^*)+\nabla V^{cl}_1(x)^T(Ax + B_1u^* + B_2v^*)=\rho_1 V^{cl}_1(x).
\end{aligned}
\end{equation}
In the classical approach, we often suppose that the value function takes the form
\begin{equation}\label{4.5}
V^{cl}_k(x)=x^Tw_kx,\,k=1,2.
\end{equation}
Substituting \eqref{4.5} into \eqref{4.2} and \eqref{4.3} yields the following coupled Riccati equations
\begin{equation}\label{4.6}
\left\{\begin{array}{ll}
\rho_1w_1=Q_1+\hat{u}^TR_1\hat{u}+2\theta_1\hat{v}^TR_1\hat{u}+\theta_1^2\hat{v}^TR_1\hat{v}\\
\qq\q+2w_1A+2w_1B_1\hat{u}+2w_1B_2\hat{v},\\
\rho_2w_2=Q_2+\hat{v}^TR_2\hat{v}+2\theta_2\hat{u}^TR_2\hat{v}+\theta_2^2\hat{u}^TR_2\hat{u}\\
\qq\q+2w_2A+2w_2B_1\hat{u}+2w_2B_2\hat{v},\\
\end{array} \right.
\end{equation}
where
\begin{equation}
\begin{aligned}
\hat{u}=&\frac{\theta_1R_2^{-1} B_2^T w_2}{(1 - \theta_1 \theta_2)}-\frac{R_1^{-1} B_1^T w_1}{(1 - \theta_1 \theta_2)^2}+\frac{\theta_2R_1^{-1} B_1^T w_1}{(1-\theta_1 \theta_2)^2},\\
\hat{v}=&-\theta_2 \hat{u}-R_2^{-1}B^T_2w_2.
\end{aligned}
\end{equation}
Equation \eqref{4.6} is a system of nonlinear equations, and its solutions are often non-unique. Set
\begin{equation}
\begin{aligned}
R_1 =& R_2 = 1.0, \quad B_1 = 2.0, \quad B_2 = 1.0,\\
\theta_1 =& 0.6, \quad\theta_2 = 0.4,\quad A = 0.1,\\
 Q_1 =& 0.2, \quad Q_2 = 0.1,  \quad \rho_1 =1.5, \quad \rho_2 = 1.6.
\end{aligned}
\end{equation}
Equation \eqref{4.6} degenerates to
\begin{equation}
\left\{\begin{aligned}
&\frac{1500}{361} w_1^2 - \frac{10}{19} w_1 w_2 + \frac{13}{10} w_1 - \frac{1}{5} = 0,\\[8pt]
&\frac{2400}{361} w_1 w_2 - \frac{29}{19} w_2^2 + \frac{7}{5} w_2 - \frac{1}{10} = 0.
\end{aligned} \right.
\end{equation}
In this case, there are two sets of positive solutions
\begin{equation}
\begin{aligned}
(w_1, w_2) \approx& (0.1143037,0.0479209),\\
(w_1, w_2) \approx& (0.1724645,1.6282089).
\end{aligned}
\end{equation}
It is obvious that the first set is the optimal solution, while the second set is not. This also indicates that classical methods are prone to falling into local optima during the solving process.  At this point, the classical value function is
\begin{equation}
V^{cl}_1(x)=0.1143037x^2,\quad V^{cl}_2(x)=0.0479209x^2.
\end{equation}

In contrast, we use the PIA method to solve this problem within the RL framework. Specifically, we consider a single-regime scenario (i.e., the state space of the Markov chain is $\mathcal{M} = \{1\}$) and set $C = D_1 = D_2 = 0$. The activation function is $\varphi_1(x) = \varphi_2(x) = \begin{bmatrix} x^2 & x & 1 \end{bmatrix}^T$, where $\lambda_1 = 0.01$ and $\lambda_2 = 0.01$. The algorithm starts from $W_1^{(0)} = \begin{bmatrix} 1 & 1 & 1 \end{bmatrix}^T$, $W_2^{(0)} = \begin{bmatrix} 1 & 1 & 1 \end{bmatrix}^T$, $\alpha^{(0)} = 0$, $\beta^{(0)} = 0$, and the initial state is $x = 0.1$ under regime $i = 1$. One point is sampled every 0.05 seconds, and after collecting $M = 1000$ points, we obtain the following results.
\begin{figure}[h!]
	\centering
	\includegraphics[width=3.0in]{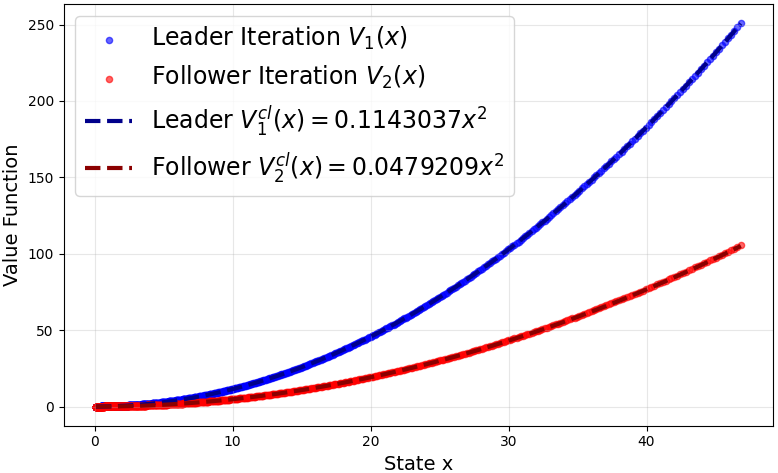}
	\caption{Value functions $V_1(x)$, $V_2(x)$, $V^{cl}_1(x)$, $V^{cl}_2(x)$.}
	\label{fig:img7}
\end{figure}
\begin{table}[h!]
\caption{Final weight vectors of the critic NNs\label{table:1}}
\centering
\begin{tabular}{|c|ccc|}
\hline
 & $W_{k1}$ & $W_{k2}$ & $W_{k3}$ \\
\hline
$k=1$ & $1.14419 \times 10^{-1}$ & $-1.17929 \times 10^{-10}$ & $1.27348 \times 10^{-2}$ \\
\hline
$k=2$ & $4.81057 \times 10^{-2}$ & $3.05420 \times 10^{-11}$ & $1.13139 \times 10^{-2}$ \\
\hline
\end{tabular}
\end{table}

Fig. \ref{fig:img7} displays the curve of the value function, while Table \ref{table:1} presents the final weight vectors of the critic NNs. The results indicate that with a low temperature parameter, the PIA can achieve a good approximation to the optimal value function, which shows that the RL framework plays a crucial role in preventing the algorithm from converging to local optima.

Mathematically, classical deterministic methods greedily exploit the current best solution and are highly prone to getting trapped in local optima. In contrast, the introduction of entropy regularization fosters a randomized policy. This mechanism maintains a certain degree of stochasticity, encouraging the algorithm to continuously explore the action space rather than prematurely committing to a local peak. By balancing exploration and exploitation, the randomized policy enables the system to evaluate a broader range of state-action trajectories, successfully escape the attraction of local optima, and ultimately discover the global optimal solution.

\section{Conclusion}
This paper investigates the LQ-SDG problem subject to Markovian regime switching within the ERRL framework. It begins by introducing the ERRL framework under regime-switching diffusions and clarifying the equilibrium policy pair problem. We then characterize the players' regime-dependent equilibrium policy pairs through the game's weakly-coupled value functions, demonstrating that these policies can effectively solve the formulated LQ-SDG problem. Subsequently, a policy improvement algorithm (PIA) is proposed to approximate the value functions and equilibrium policies, with a critic neural network architecture introduced to address the high-dimensional computational challenges inherent in multi-regime systems. Finally, the effectiveness of the proposed algorithms is verified through multiple numerical examples, while the model's convergence, the impact of temperature parameters, and robustness to structural shifts are analyzed. Specifically, we use a practical example to illustrate how this exploratory framework effectively avoids the \lq\lq local optimum trap\rq\rq.

In addition, the game framework constructed in this paper demonstrates excellent scalability: on one hand, its applicability can be extended from the current linear-quadratic regime-switching setting to nonlinear systems, and can also be expanded to multi-agent interaction scenarios under Markovian environments (a classical case is provided in \cite{li2020hierarchical}); on the other hand, this framework provides further space for algorithm optimization, such as exploring fully model-free RL algorithms that do not require prior knowledge of the control input matrices.

\appendices
\section{Exploratory dynamic system with regime switching}\label{Appendix A}
\setcounter{equation}{0}
\renewcommand{\theequation}{A.\arabic{equation}}
Consider the dynamic system with regime switching \eqref{2.6}, and let $x_s^{k,l}$, $k=1,\cdots,N, l=1,\cdots,M$, be the copies of the dynamic system respectively under the policies $(u_s^k,v_s^l)$, each independently sampled from $(\boldsymbol{\pi}_s, \boldsymbol{\gamma}_s)$ conditional on the current state $x_s$ and regime $\alpha_s$. Let $d s > 0$ be small enough and assume that the current regime $\alpha_s$, as well as the corresponding policies $(u_s^k,v_s^l)$, remain fixed from $s$ to $s+d s$. Then, the increments of these dynamic system copies are, for $k=1,\cdots,N, l=1,\cdots,M$,
\begin{equation*}\label{A.1}
\begin{aligned}
d x_s^{k,l} \equiv& x_{s+d s}^{k,l} - x_s^{k,l}\\
\approx& b(x_s^{k,l}, \alpha_s, u_s^k, v_s^l)d s\\
&+\sigma(x_s^{k,l}, \alpha_s, u_s^k, v_s^l)\left(W_{s+d s}^{k,l}- W_s^{k,l}\right),\, s\geq0.
\end{aligned}
\end{equation*}
Each such dynamic system $x_s^{k,l}, k=1,\cdots,N, l=1,\cdots,M$, can be viewed as an independent sample from the exploratory dynamic system $X_s^{\boldsymbol{\pi},\boldsymbol{\gamma}}$.

By conditioning on the current regime $\alpha_s = i \in \mathcal{M}$, it can be concluded from the law of large numbers that,
\begin{equation*}
\begin{aligned}\label{A.2}
&\lim _{N \rightarrow \infty}\lim _{M \rightarrow \infty}\frac{1}{N}\frac{1}{M} \sum_{k=1}^N\sum_{l=1}^M d x_s^{k,l}\\
\approx&\lim _{N \rightarrow \infty}\lim _{M \rightarrow \infty}\frac{1}{N}\frac{1}{M}\sum_{k=1}^N\sum_{l=1}^M b(x_s^{k,l}, i, u_s^k, v_s^l)ds\\
&+\lim _{N \rightarrow \infty}\lim _{M \rightarrow \infty}\frac{1}{N}\frac{1}{M} \sum_{k=1}^N\sum_{l=1}^M \\
&\quad \sigma(x_s^{k,l}, i, u_s^k, v_s^l)\left(W_{s+d s}^{k,l}- W_s^{k,l}\right)\\
\xrightarrow{\text{a.s.}}&\left\{\int_{\mathbb{R}^p}\left[\int_{\mathbb{R}^p}b(x_s, i, u, v)\boldsymbol{\pi}(u)\,du\right]\boldsymbol{\gamma}(v)\,dv\right\}d s,
\end{aligned}
\end{equation*}
and
\begin{equation*}\label{A.3}
\begin{aligned}
&\lim _{N \rightarrow \infty}\lim _{M \rightarrow \infty}\frac{1}{N}\frac{1}{M} \sum_{k=1}^N\sum_{l=1}^M d x_s^{k,l}\left(d x_s^{k,l}\right)^T\\
\approx&\lim _{N \rightarrow \infty}\lim _{M \rightarrow \infty}\frac{1}{N}\frac{1}{M} \sum_{k=1}^N\sum_{l=1}^M\\
&\sigma(x_s^{k,l}, i, u_s^k, v_s^l)\sigma^T(x_s^{k,l}, i, u_s^k, v_s^l)ds\xrightarrow{\text{a.s.}}\\
&\left\{\int_{\mathbb{R}^p}\left[\int_{\mathbb{R}^p}\sigma(x_s, i, u, v)\sigma^T(x_s, i, u, v)\boldsymbol{\pi}(u)du\right]\boldsymbol{\gamma}(v)dv\right\}d s.
\end{aligned}
\end{equation*}

Since each state $x_s^{k,l}$ is an independent sample from $X_s^{\boldsymbol{\pi},\boldsymbol{\gamma}}$, the increments $dx_s^{k,l}$ are independent samples from $d X_s^{\boldsymbol{\pi},\boldsymbol{\gamma}}$. Therefore, using the law of large numbers again, we obtain
\begin{equation}\label{A.4}
\lim _{N \rightarrow \infty}\lim _{M \rightarrow \infty}\frac{1}{N}\frac{1}{M} \sum_{k=1}^N\sum_{l=1}^M d x_s^{k,l}=d X_s^{\boldsymbol{\pi},\boldsymbol{\gamma}}.
\end{equation}
Thus, under the regime $\alpha_s = i$, the exploratory drift and diffusion are identified as
\begin{equation*}
\tilde{b}(x, i, \boldsymbol{\pi},\boldsymbol{\gamma}):=\int_{\mathbb{R}^p}\left[\int_{\mathbb{R}^p}b(x, i, u, v)\boldsymbol{\pi}(u)\,du\right]\boldsymbol{\gamma}(v)\,dv,
\end{equation*}
\begin{equation*}
\begin{aligned}
 \;&\tilde{\sigma}(x, i, \boldsymbol{\pi},\boldsymbol{\gamma}):=\\
 &\left\{\int_{\mathbb{R}^p}\left[\int_{\mathbb{R}^p}\sigma(x, i, u, v)\sigma^T(x, i, u, v)\boldsymbol{\pi}(u)du\right]\boldsymbol{\gamma}(v)dv\right\}^{1/2}.
\end{aligned}
\end{equation*}
Similarly, utilizing the law of large numbers, the exploratory performance functions under regime $i$ are
\begin{equation*}
\tilde{r}_1(x, i, \boldsymbol{\pi}, \boldsymbol{\gamma}):=\int_{\mathbb{R}^p}\left[\int_{\mathbb{R}^p}r_1(x, i, u, v)\boldsymbol{\pi}(u)\,du\right]\boldsymbol{\gamma}(v)\,dv,
\end{equation*}
\begin{equation*}
\tilde{r}_2(x, i, \boldsymbol{\pi}, \boldsymbol{\gamma}):=\int_{\mathbb{R}^p}\left[\int_{\mathbb{R}^p}r_2(x, i, u, v) \boldsymbol{\pi}(u)\,du\right]\boldsymbol{\gamma}(v)\,dv.
\end{equation*}
\section{Proof of Proposition \ref{Proposition 1}}\label{Appendix B}
\setcounter{equation}{0}
\renewcommand{\theequation}{B.\arabic{equation}}
\begin{proof}
For any $\boldsymbol{\pi}\in\mathcal{U}$, $\boldsymbol{\gamma}\in\mathcal{V}$ and fixed regime $i \in \mathcal{M}$, we have
\begin{itemize}
\item[($1$)]  Lipschitz condition:
\begin{equation*}
\begin{aligned}
&\left|\left|\tilde{b}(x, i, \boldsymbol{\pi},\boldsymbol{\gamma})-\tilde{b}(y, i, \boldsymbol{\pi},\boldsymbol{\gamma})\right|\right|\\
=&\left|\left|\int_{\mathbb{R}^p}\left[\int_{\mathbb{R}^p}b(x, i, u,v)-b(y, i, u,v)\boldsymbol{\pi}(u)\,du\right]\boldsymbol{\gamma}(v)\,dv\right|\right|\\
\leq&\int_{\mathbb{R}^p}\left[\int_{\mathbb{R}^p}\left|\left|b(x, i, u,v)-b(y, i, u,v)\right|\right|\boldsymbol{\pi}(u)\,du\right]\boldsymbol{\gamma}(v)\,dv\\
\leq&\int_{\mathbb{R}^p}\left[\int_{\mathbb{R}^p}\left|\left|A(i)x-A(i)y\right|\right|\boldsymbol{\pi}(u)\,du\right]\boldsymbol{\gamma}(v)\,dv\\
\leq&\left|\left|A(i)\right|\right|\left|\left|x-y\right|\right|.
\end{aligned}
\end{equation*}
Similarly,
\begin{equation*}
\left|\left|\tilde{\sigma}(x, i, \boldsymbol{\pi},\boldsymbol{\gamma})-\tilde{\sigma}(y, i, \boldsymbol{\pi},\boldsymbol{\gamma})\right|\right|\leq\left|\left|C(i)\right|\right|\left|\left|x-y\right|\right|.
\end{equation*}
\item[($2$)] Linear growth condition:
\begin{equation*}
\begin{aligned}
&\left|\left|\tilde{b}(x, i, \boldsymbol{\pi},\boldsymbol{\gamma})\right|\right|\\
=& \left|\left|\int_{\mathbb{R}^p}\left[\int_{\mathbb{R}^p}\left(A(i)x + B_1(i)u + B_2(i)v\right)\boldsymbol{\pi}(u)\,du\right]\boldsymbol{\gamma}(v)\,dv\right|\right|\\
\leq&\int_{\mathbb{R}^p}\left[\int_{\mathbb{R}^p}\left|\left|A(i)x + B_1(i)u + B_2(i)v\right|\right|\boldsymbol{\pi}(u)\,du\right]\boldsymbol{\gamma}(v)\,dv\\
\leq&\int_{\mathbb{R}^p}\left[\int_{\mathbb{R}^p}\left|\left|B_1(i)u + B_2(i)v\right|\right|\boldsymbol{\pi}(u)\,du\right]\boldsymbol{\gamma}(v)\,dv\\
&+\left|\left|A(i)\right|\right|\left|\left|x\right|\right|.
\end{aligned}
\end{equation*}
Similarly, we can prove $\tilde{\sigma}(x, i, \boldsymbol{\pi},\boldsymbol{\gamma})$ has linear growth.
\end{itemize}
By the classical existence and uniqueness theorem for stochastic differential equations (c.f. \cite{oksendal2013stochastic}, Theorem 5.2.1), \eqref{2.11} has a unique strong solution.
\end{proof}

\section{Proof of Lemma \ref{Lemma 1}}\label{Appendix C}
\setcounter{equation}{0}
\renewcommand{\theequation}{C.\arabic{equation}}
\begin{proof}
To simplify the notation, we denote $X_s^{\boldsymbol{\pi},\boldsymbol{\gamma}^*(\boldsymbol{\pi})}:=X_s$. According to Definition \ref{Definition 2},
\begin{equation*}
V_1(x, i) = \min_{\boldsymbol{\pi}\in\mathcal{U}} J_1(x, i, \boldsymbol{\pi},\boldsymbol{\gamma}^*(\boldsymbol{\pi})).
\end{equation*}
Let $h>0$ be arbitrary and applying dynamic programming principle,
\begin{equation}\label{C.2}
\begin{aligned}
V_1(x, i)&=\min_{\boldsymbol{\pi}\in\mathcal{U}}\e_{x,i}\Bigg[\int_{0}^{h}e^{-\rho_{1}s}\Big[\tilde{r}_1(X_s, \alpha_s, \boldsymbol{\pi}_s,\boldsymbol{\gamma}^*_s(\boldsymbol{\pi}))\\
& +\lambda_1\int_{\mathbb{R}^p} \boldsymbol{\pi}_s(u)\ln\boldsymbol{\pi}_s(u)\,du\Big]ds+e^{-\rho_1h}V_1(X_{h}, \alpha_h)\Bigg].
\end{aligned}
\end{equation}
Applying the generalized It\^{o}'s formula with Markovian switching to $e^{-\rho_{1}h}V_1(X_{h}, \alpha_h)$ and integrating from $0$ to $h$, we have
\begin{equation}\label{C.3}
\begin{aligned}
e^{-\rho_{1}h}&V_1(X_{h}, \alpha_h)-V_1(x, i)=\int_{0}^{h}e^{-\rho_{1}s}\Big[-\rho_{1}V_1(X_{s}, \alpha_s)\\
&+\nabla V_1(X_s, \alpha_s)^{T}\tilde{b}(X_s, \alpha_s, \boldsymbol{\pi}_s,\boldsymbol{\gamma}^*_s(\boldsymbol{\pi}))\\
&+\frac{1}{2}tr\left(\tilde{\sigma}(X_s, \alpha_s, \boldsymbol{\pi}_s,\boldsymbol{\gamma}^*_s(\boldsymbol{\pi}))\tilde{\sigma}^T(\cdots)\Delta V_1(X_s, \alpha_s)\right)\\
&+\sum_{j=1}^m q_{\alpha_s j}V_1(X_s, j)\Big]ds+ M_h \\
&+\int_{0}^{h}e^{-\rho_{1}s}\nabla V_1(X_s, \alpha_s)^{T}\tilde{\sigma}(X_s, \alpha_s, \boldsymbol{\pi}_s,\boldsymbol{\gamma}^*_s(\boldsymbol{\pi}))dW_s ,
\end{aligned}
\end{equation}
where $M_h$ is a zero-mean martingale arising from the pure jump process of the Markov chain. Substituting equation \eqref{C.3} back into equation \eqref{C.2}, taking expectation (where the integrals with respect to $dW_s$ and the martingale $M_h$ vanish), yields
\begin{equation*}
\begin{aligned}
\min_{\boldsymbol{\pi}\in\mathcal{U}}\e_{x,i}&\Bigg[\int_{0}^{h}e^{-\rho_{1}s}\Bigg[\tilde{r}_1(X_s, \alpha_s, \boldsymbol{\pi}_s,\boldsymbol{\gamma}^*_s(\boldsymbol{\pi}))\\
&+\lambda_1\int_{\mathbb{R}^p} \boldsymbol{\pi}_s(u)\ln\boldsymbol{\pi}_s(u)\,du-\rho_{1}V_1(X_{s}, \alpha_s)\\
&+\nabla V_1(X_s, \alpha_s)^{T}\tilde{b}(X_s, \alpha_s, \boldsymbol{\pi}_s,\boldsymbol{\gamma}^*_s(\boldsymbol{\pi}))\\
&+\frac{1}{2}tr\left(\tilde{\sigma}(X_s, \alpha_s, \boldsymbol{\pi}_s,\boldsymbol{\gamma}^*_s(\boldsymbol{\pi}))\tilde{\sigma}^T(\cdots)\Delta V_1(X_s, \alpha_s)\right)\\
&+\sum_{j=1}^m q_{\alpha_s j}V_1(X_s, j)\Bigg]ds\Bigg]=0.
\end{aligned}
\end{equation*}
Let $h\rightarrow 0$, divide by $h$, and apply the mean value theorem for integrals, and we obtain
\begin{equation*}
\begin{aligned}
\rho_1V_1(x, i)=&\min_{\boldsymbol{\pi}\in\mathcal{U}}\Big[\tilde{r}_1(x, i, \boldsymbol{\pi}, \boldsymbol{\gamma}^*(\boldsymbol{\pi}))+\lambda_1\int_{\mathbb{R}^p} \boldsymbol{\pi}(u)\ln \boldsymbol{\pi}(u)\,du\\
&+ \nabla V_1(x, i)^T\tilde{b}(x, i, \boldsymbol{\pi},\boldsymbol{\gamma}^*(\boldsymbol{\pi}))\\
&+\frac{1}{2}tr\left(\tilde{\sigma}(x, i, \boldsymbol{\pi},\boldsymbol{\gamma}^*(\boldsymbol{\pi}))\tilde{\sigma}^T(\cdots)\Delta V_1(x, i)\right)\\
&+\sum_{j=1}^m q_{ij}V_1(x, j)\Big].
\end{aligned}
\end{equation*}
We often rewrite this PDE in the following form
\begin{equation*}
\rho_1\phi_1(x, i)=\min_{\boldsymbol{\pi}\in\mathcal{U}}\mathcal{H}_1(x, i, \nabla \phi_1(x, i),\Delta \phi_1(x, i), \boldsymbol{\pi}, \boldsymbol{\gamma}^*(\boldsymbol{\pi})),
\end{equation*}
where $\phi_1$ is the solution to this PDE, and the specific form of $\mathcal{H}_1$ incorporates the regime switching jumps:
\begin{equation*}
\begin{aligned}
&\mathcal{H}_1(x, i, \nabla \phi_1(x, i),\Delta \phi_1(x, i), \boldsymbol{\pi}, \boldsymbol{\gamma}^*(\boldsymbol{\pi}))\\
:=&\tilde{r}_1(x, i, \boldsymbol{\pi}, \boldsymbol{\gamma}^*(\boldsymbol{\pi}))+\lambda_1\int_{\mathbb{R}^p} \boldsymbol{\pi}(u)\ln \boldsymbol{\pi}(u)\,du\\
&+\nabla \phi_1(x, i)^T\tilde{b}(x, i, \boldsymbol{\pi},\boldsymbol{\gamma}^*(\boldsymbol{\pi}))+\sum_{j=1}^m q_{ij}\phi_1(x, j)\\
&+\frac{1}{2}tr\left(\tilde{\sigma}(x, i, \boldsymbol{\pi},\boldsymbol{\gamma}^*(\boldsymbol{\pi}))\tilde{\sigma}^T(\cdots)\Delta \phi_1(x, i)\right).
\end{aligned}
\end{equation*}

We can also prove that if $V_2(\cdot, i; \boldsymbol{\pi}) \in C^{2}(\mathbb{R}^n)$ for each $i \in \mathcal{M}$, then $V_2(x, i; \boldsymbol{\pi})$ is a solution to the HJBI equation \eqref{3.2}. The proof is omitted.
\end{proof}

\section{Proof of Theorem \ref{Theorem 3}}\label{Appendix D}
\setcounter{equation}{0}
\renewcommand{\theequation}{D.\arabic{equation}}
\begin{proof}
To simplify the notation, we denote $X_s^{\boldsymbol{\pi},\boldsymbol{\gamma}^*(\boldsymbol{\pi})}:=X_s$. Suppose that $\boldsymbol{\pi}$ is an arbitrarily given policy for leader, and $\boldsymbol{\gamma}^*(\boldsymbol{\pi})$ is the  corresponding optimal response policy for follower.

Applying the generalized It\^{o}'s formula to $e^{-\rho_1 s}\phi_1(X_s, \alpha_s)$, we get
\begin{equation*}
\begin{aligned}
&d e^{-\rho_{1}s}\phi_1(X_s, \alpha_s)=\\
&e^{-\rho_{1}s}\Big[-\rho_{1}\phi_1(X_s, \alpha_s)+\nabla \phi_1(X_s, \alpha_s)^{T}\tilde{b}(X_s, \alpha_s, \boldsymbol{\pi}_s,\boldsymbol{\gamma}_s^*(\boldsymbol{\pi}))\\
&+\frac{1}{2}tr\left(\tilde{\sigma}(\cdot)\tilde{\sigma}^T(\cdot)\Delta \phi_1(X_s, \alpha_s)\right)+\sum_{j=1}^m q_{\alpha_s j}\phi_1(X_s, j)\Big]ds \\
&+e^{-\rho_{1}s}\nabla \phi_1(X_s, \alpha_s)^{T}\tilde{\sigma}(X_s, \alpha_s, \boldsymbol{\pi}_s,\boldsymbol{\gamma}_s^*(\boldsymbol{\pi}))dW_s + dM_s,
\end{aligned}
\end{equation*}
where $M_s$ is the associated jump martingale. Let $T > 0$ be arbitrary, and define the stopping times
\begin{equation*}
\begin{aligned}
\tau_{n}:=&\Big\{t\geq0:\int_{0}^{t}||e^{-\rho_1 s}\nabla \phi_1(X_s, \alpha_s)^{T}\\
&\tilde{\sigma}(X_s, \alpha_s, \boldsymbol{\pi}_s,\boldsymbol{\gamma}_s^*(\boldsymbol{\pi}))||^{2}ds\geq n\Big\},\, n\geq1.
\end{aligned}
\end{equation*}
Integrating from $0$ to $T\wedge\tau_{n}$ and taking expectation $\e_{x,i}[\cdot]$ on both sides, we have
\begin{equation*}
\begin{aligned}
&\e_{x,i}\left[e^{-\rho_1(T\wedge\tau_{n})}\phi_1(X_{T\wedge\tau_{n}}, \alpha_{T\wedge\tau_{n}})\right]-\phi_1(x, i)\\
=&\e_{x,i}\Bigg[\int_{0}^{T\wedge\tau_{n}}e^{-\rho_{1}s}\Big[-\rho_{1}\phi_1(X_s, \alpha_s)\\
&+\nabla \phi_1(X_s, \alpha_s)^{T}\tilde{b}(X_s, \alpha_s, \boldsymbol{\pi}_s,\boldsymbol{\gamma}_s^*(\boldsymbol{\pi}))\\
&+\frac{1}{2}tr(\tilde{\sigma}(\cdot)\tilde{\sigma}^T(\cdot)\Delta \phi_1(X_s, \alpha_s))+\sum_{j=1}^m q_{\alpha_s j}\phi_1(X_s, j)\Big]ds\Bigg].
\end{aligned}
\end{equation*}
Setting $n \to \infty$, we deduce that
\begin{equation*}
\begin{aligned}
&\e_{x,i}\left[e^{-\rho_1T}\phi_1(X_{T}, \alpha_T)\right]-\phi_1(x, i)\\
=&\e_{x,i}\Bigg[\int_{0}^{T}e^{-\rho_{1}s}\Big[-\rho_{1}\phi_1(X_s, \alpha_s)\\
&+\nabla \phi_1(X_s, \alpha_s)^{T}\tilde{b}(X_s, \alpha_s, \boldsymbol{\pi}_s,\boldsymbol{\gamma}_s^*(\boldsymbol{\pi}))\\
&+\frac{1}{2}tr(\tilde{\sigma}(\cdot)\tilde{\sigma}^T(\cdot)\Delta \phi_1(X_s, \alpha_s))+\sum_{j=1}^m q_{\alpha_s j}\phi_1(X_s, j)\Big]ds\Bigg].
\end{aligned}
\end{equation*}
Next, using the condition (iv) in Assumption \ref{assumption 1} and applying the dominated convergence theorem yields
\begin{equation*}
\begin{aligned}
-\phi_1(x, i)=&\e_{x,i}\Bigg[\int_{0}^{\infty}e^{-\rho_{1}s}\Big[-\rho_{1}\phi_1(X_s, \alpha_s)\\
&+\nabla \phi_1(X_s, \alpha_s)^{T}\tilde{b}(X_s, \alpha_s, \boldsymbol{\pi}_s,\boldsymbol{\gamma}_s^*(\boldsymbol{\pi}))\\
&+\frac{1}{2}tr(\tilde{\sigma}(X_s, \alpha_s, \boldsymbol{\pi}_s, \boldsymbol{\gamma}_s^*(\boldsymbol{\pi}))\tilde{\sigma}^T(\dots)\Delta \phi_1(X_s, \alpha_s))\\
&+\sum_{j=1}^m q_{\alpha_s j}\phi_1(X_s, j)\Big]ds\Bigg].
\end{aligned}
\end{equation*}
Adding the same item
\begin{equation*}
\begin{aligned}
&\e_{x,i}\left[\int_0^{\infty}e^{-\rho_1 s} \left[\tilde{r}_1(X_s, \alpha_s, \boldsymbol{\pi}_s, \boldsymbol{\gamma}^*_s(\boldsymbol{\pi}))\right.\right.\\
&\qquad\left.\left.+\lambda_1\int_{\mathbb{R}^p} \boldsymbol{\pi}_s(u)\ln \boldsymbol{\pi}_s(u)\,du\,\right] ds\right]
\end{aligned}
\end{equation*}
on both sides, with the definition of the exploratory performance function \eqref{2.12} of leader, we have
\begin{equation*}
\begin{aligned}
&J_1(x, i, \boldsymbol{\pi},\boldsymbol{\gamma}^*(\boldsymbol{\pi}))-\phi_1(x, i)\\
=&\e_{x,i}\Big[\int_{0}^{\infty}e^{-\rho_{1}s}[-\rho_{1}\phi_1(X_s, \alpha_s)\\
&+\mathcal{H}_1(X_s, \alpha_s, \nabla \phi_1(X_s, \alpha_s) ,\Delta \phi_1(X_s, \alpha_s), \boldsymbol{\pi}_s, \boldsymbol{\gamma}_s^*(\boldsymbol{\pi}))]ds\Big].
\end{aligned}
\end{equation*}
On the other hand,
\begin{equation*}
\begin{aligned}
&\rho_1\phi_1(X_s^{\boldsymbol{\pi}^*, \boldsymbol{\gamma}^*}, \alpha_s)\\
=&\mathcal{H}_1(X_s^{\boldsymbol{\pi}^*, \boldsymbol{\gamma}^*}, \alpha_s, \nabla \phi_1(X_s^{\boldsymbol{\pi}^*, \boldsymbol{\gamma}^*}, \alpha_s) ,\Delta \phi_1(X_s^{\boldsymbol{\pi}^*, \boldsymbol{\gamma}^*}, \alpha_s), \boldsymbol{\pi}_s^*, \boldsymbol{\gamma}_s^*)\\
\leq&\mathcal{H}_1(X_s, \alpha_s, \nabla \phi_1(X_s, \alpha_s) ,\Delta \phi_1(X_s, \alpha_s), \boldsymbol{\pi}_s, \boldsymbol{\gamma}_s^*(\boldsymbol{\pi})), \,\forall \boldsymbol{\pi}\in\mathcal{U}.
\end{aligned}
\end{equation*}
Therefore,
\begin{equation*}
\begin{aligned}
&J_1(x, i, \boldsymbol{\pi},\boldsymbol{\gamma}^*(\boldsymbol{\pi}))-\phi_1(x, i)\\
\geq&\e_{x,i}\Big[\int_{0}^{\infty}e^{-\rho_{1}s}[-\rho_{1}\phi_1(X_s^{\boldsymbol{\pi}^*, \boldsymbol{\gamma}^*}, \alpha_s)\\
&+\mathcal{H}_1(X_s^{\boldsymbol{\pi}^*, \boldsymbol{\gamma}^*}, \alpha_s, \cdots)]ds\Big]\\
=&0=J_1(x, i, \boldsymbol{\pi}^*,\boldsymbol{\gamma}^*)-\phi_1(x, i).
\end{aligned}
\end{equation*}
The first and second equations of Theorem \ref{Theorem 3} are proved.
The proofs proces for the third and fourth equations of Theorem \ref{Theorem 3} are similar and  omitted here.
\end{proof}

\section{Proof of Theorem \ref{Theorem 5}}\label{Appendix E}
\setcounter{equation}{0}
\renewcommand{\theequation}{E.\arabic{equation}}
\begin{proof}
Let $X_{s}^{\boldsymbol{\gamma}}$ (resp., $X_{s}^{\tilde{\boldsymbol{\gamma}}}$), $s \in [0, \infty)$ be the state process under policy pair $(\boldsymbol{\pi},\boldsymbol{\gamma}(\boldsymbol{\pi}))$ (resp., $(\boldsymbol{\pi},\tilde{\boldsymbol{\gamma}}(\boldsymbol{\pi}))$). First applying the generalized It\^{o}'s formula,
\begin{equation*}
\begin{aligned}
&de^{-\rho_{2}s}J_2(X_s^{\boldsymbol{\gamma}}, \alpha_s, \boldsymbol{\pi}_s,\boldsymbol{\gamma}_s(\boldsymbol{\pi}))\\
=&e^{-\rho_{2}s}\Big[ \nabla J_2(X_s^{\boldsymbol{\gamma}}, \alpha_s, \boldsymbol{\pi}_s,\boldsymbol{\gamma}_s(\boldsymbol{\pi}))^{T}\tilde{b}(X_s^{\boldsymbol{\gamma}}, \alpha_s, \boldsymbol{\pi}_s,\boldsymbol{\gamma}_s(\boldsymbol{\pi}))\\
&+\frac{1}{2}tr\left(\tilde{\sigma}(X_s^{\boldsymbol{\gamma}}, \alpha_s, \boldsymbol{\pi}_s,\boldsymbol{\gamma}_s(\boldsymbol{\pi}))\tilde{\sigma}^T(X_s^{\boldsymbol{\gamma}}, \alpha_s, \boldsymbol{\pi}_s, \boldsymbol{\gamma}_s(\boldsymbol{\pi}))\right.\\
&\quad\left.\Delta J_2(X_s^{\boldsymbol{\gamma}}, \alpha_s, \boldsymbol{\pi}_s,\boldsymbol{\gamma}_s(\boldsymbol{\pi}))\right)-\rho_{2}J_2(X_s^{\boldsymbol{\gamma}}, \alpha_s, \boldsymbol{\pi}_s,\boldsymbol{\gamma}_s(\boldsymbol{\pi}))\\
&+\sum_{j=1}^m q_{\alpha_s j}J_2(X_s^{\boldsymbol{\gamma}}, j, \boldsymbol{\pi}_s,\boldsymbol{\gamma}_s(\boldsymbol{\pi}))\Big]ds + dM_s \\
&+e^{-\rho_{2}s}\nabla J_2(X_s^{\boldsymbol{\gamma}}, \alpha_s, \boldsymbol{\pi}_s,\boldsymbol{\gamma}_s(\boldsymbol{\pi}))^{T}\tilde{\sigma}(X_s^{\boldsymbol{\gamma}}, \alpha_s, \boldsymbol{\pi}_s,\boldsymbol{\gamma}_s(\boldsymbol{\pi}))dW_s.
\end{aligned}
\end{equation*}
Define the stopping times
\begin{equation*}
\begin{aligned}
\tau_{n}:=&\Big\{t\geq0:\int_{0}^{t}||e^{-\rho_2s}\nabla J_2(X_s^{\boldsymbol{\gamma}}, \alpha_s, \boldsymbol{\pi}_s,\boldsymbol{\gamma}_s(\boldsymbol{\pi}))^{T}\\
&\tilde{\sigma}(X_s^{\boldsymbol{\gamma}}, \alpha_s, \boldsymbol{\pi}_s,\boldsymbol{\gamma}_s(\boldsymbol{\pi}))||^{2}ds\geq n\Big\}
\end{aligned}
\end{equation*}
for $n\geq1$. Integrating from $0$ to $\tau_{n}$ and taking expectation $\e_{x,i}[\cdot]$ on both sides, we have
\begin{equation*}
\begin{aligned}
&J_2(x, i, \boldsymbol{\pi},\boldsymbol{\gamma}(\boldsymbol{\pi}))\\
=&\e_{x,i}\Bigg[-\int_{0}^{\tau_{n}}e^{-\rho_{2}s}\Big[ \nabla J_2(\dots)^{T}\tilde{b}(X_s^{\boldsymbol{\gamma}}, \alpha_s, \boldsymbol{\pi}_s,\boldsymbol{\gamma}_s(\boldsymbol{\pi}))\\
&+\frac{1}{2}tr\left(\tilde{\sigma}(X_s^{\boldsymbol{\gamma}}, \alpha_s, \boldsymbol{\pi}_s,\boldsymbol{\gamma}_s(\boldsymbol{\pi}))\tilde{\sigma}^T(\dots)\Delta J_2(\dots)\right)\\
&-\rho_{2}J_2(X_s^{\boldsymbol{\gamma}}, \alpha_s, \boldsymbol{\pi}_s,\boldsymbol{\gamma}_s(\boldsymbol{\pi}))+\sum_{j=1}^m q_{\alpha_s j}J_2(\cdots)\Big]ds \\
&+e^{-\rho_2\tau_{n}}J_2(X_{\tau_{n}}^{\boldsymbol{\gamma}}, \alpha_{\tau_n}, \boldsymbol{\pi}_{\tau_{n}},\boldsymbol{\gamma}_{\tau_{n}}(\boldsymbol{\pi}))\Bigg].
\end{aligned}
\end{equation*}
On the other hand,
\begin{equation}\label{5.7}
\begin{aligned}
&\rho_2J_2(x, i, \boldsymbol{\pi},\boldsymbol{\gamma}(\boldsymbol{\pi}))\\=&\mathcal{H}_2(x, i, \nabla J_2(x, i, \boldsymbol{\pi},\boldsymbol{\gamma}(\boldsymbol{\pi})) ,\Delta J_2(x, i, \boldsymbol{\pi},\boldsymbol{\gamma}(\boldsymbol{\pi})), \boldsymbol{\pi}, \boldsymbol{\gamma}(\boldsymbol{\pi}))\\
\geq&\min_{\boldsymbol{\gamma}^{\prime}}\mathcal{H}_2(x, i, \nabla J_2(x, i, \boldsymbol{\pi},\boldsymbol{\gamma}^{\prime}(\boldsymbol{\pi})) ,\Delta J_2(\dots), \boldsymbol{\pi}, \boldsymbol{\gamma}^{\prime}(\boldsymbol{\pi})).
\end{aligned}
\end{equation}
Notice that the minimizer of the Hamiltonian in \eqref{5.7} is given by the feedback policy \eqref{5.1}. It then follows that
\begin{equation*}
\begin{aligned}
&J_2(x, i, \boldsymbol{\pi},\boldsymbol{\gamma}(\boldsymbol{\pi}))\\
\geq&\e_{x,i}\Bigg[\int_{0}^{\tau_{n}}e^{-\rho_{2}s}\left[\tilde{r}_2(X_s^{\tilde{\boldsymbol{\gamma}}}, \alpha_s, \boldsymbol{\pi}_s,\tilde{\boldsymbol{\gamma}}_s(\boldsymbol{\pi}))\right.\\
&\left.+\lambda_2\int_{\mathbb{R}^p} \tilde{\boldsymbol{\gamma}}_s(\boldsymbol{\pi})(v)\ln \tilde{\boldsymbol{\gamma}}_s(\boldsymbol{\pi})(v)\,dv\right]ds\\
&+e^{-\rho_2\tau_{n}}J_2(X_{\tau_{n}}^{\boldsymbol{\gamma}}, \alpha_{\tau_n}, \boldsymbol{\pi}_{\tau_{n}},\boldsymbol{\gamma}_{\tau_{n}}(\boldsymbol{\pi}))\Bigg].
\end{aligned}
\end{equation*}
Setting $n \to \infty$ and applying the condition (iv) in Assumption \ref{assumption 1}, we obtain
\begin{equation*}
\begin{aligned}
&J_2(x, i, \boldsymbol{\pi},\boldsymbol{\gamma}(\boldsymbol{\pi}))\\
\geq&\e_{x,i}\left[\int_{0}^{\infty}e^{-\rho_{2}s}\Big[\tilde{r}_2(X_s^{\tilde{\boldsymbol{\gamma}}}, \alpha_s, \boldsymbol{\pi}_s,\tilde{\boldsymbol{\gamma}}_s(\boldsymbol{\pi}))\right.\\
&\left.+\lambda_2\int_{\mathbb{R}^p} \tilde{\boldsymbol{\gamma}}_s(\boldsymbol{\pi})(v)\ln \tilde{\boldsymbol{\gamma}}_s(\boldsymbol{\pi})(v)\,dv\right]ds\Big]=J_2(x, i, \boldsymbol{\pi},\tilde{\boldsymbol{\gamma}}(\boldsymbol{\pi})).
\end{aligned}
\end{equation*}
\end{proof}
\bibliographystyle{IEEEtran}
\bibliography{HCD2025}

\end{document}